\documentclass[11pt]{article}

\usepackage[margin=1in]{geometry}
\usepackage{natbib}

\usepackage[utf8]{inputenc} %
\usepackage[T1]{fontenc}    %
\usepackage{url}            %
\usepackage{booktabs}       %
\usepackage{amsfonts}       %
\usepackage{nicefrac}       %
\usepackage{microtype}      %
\usepackage{xcolor}         %

\usepackage[backref,colorlinks,citecolor=blue,bookmarks=true]{hyperref}

\usepackage{graphicx}
\usepackage{subcaption}
\usepackage{amsmath}
\usepackage{amssymb}
\usepackage{mathtools}
\usepackage{amsthm}

\theoremstyle{plain}
\newtheorem{theorem}{Theorem}[section]
\newtheorem{proposition}[theorem]{Proposition}
\newtheorem{lemma}[theorem]{Lemma}

\theoremstyle{definition}
\newtheorem{definition}[theorem]{Definition}

\theoremstyle{remark}
\newtheorem{remark}[theorem]{Remark}

\usepackage{thmtools}
\usepackage{thm-restate}

\usepackage{float}
\usepackage{enumitem}
\setlist{nosep}

\usepackage[capitalize,noabbrev]{cleveref}

\usepackage{tikz} 
\usetikzlibrary{patterns, decorations.pathmorphing, trees, shapes.geometric, arrows.meta}
\usepackage{wrapfig}

\newcommand*{\R}{{\mathbb{R}}}

\newcommand*{\calD}{{\mathcal{D}}}

\newcommand*{\calX}{{\mathcal{X}}}
\newcommand*{\calY}{{\mathcal{Y}}}

\newcommand*{\calP}{{\mathcal{P}}}

\newcommand*{\calT}{{\mathcal{T}}}

\newcommand*{\DY}{\Delta \calY}
\newcommand*{\DDY}{\Delta \Delta \calY}

\let\phi\varphi

\DeclareMathOperator*{\ex}{\mathbb{E}}

\let\hat\widehat

\newcommand{\ignore}[1]{}

\newcommand{\RL}{\mathrm{RL}}
\newcommand{\IL}{\mathrm{IL}}
\newcommand{\eat}[1]{}
\newcommand{\distr}{p}

\newcommand{\mix}{\pi}
\newcommand{\mixpred}{F}
\newcommand{\mixbayes}{F^*}
\newcommand{\pred}{f}

\newcommand{\pbayes}{\pred^*}

\DeclarePairedDelimiterX{\divx}[2]{(}{)}{%
  #1 , #2
}

\numberwithin{equation}{section}

\newcommand{\compareplots}[6]{
\begin{figure}[H]
    \centering
    \begin{subfigure}{0.48\textwidth}
        \centering
        \includegraphics[width=\linewidth]{#1}
        \caption{Calibrated Weak Model}
        \label{#3}
    \end{subfigure}\hfill
    \begin{subfigure}{0.48\textwidth}
        \centering
        \includegraphics[width=\linewidth]{#2}
        \caption{Uncalibrated Weak Model}
        \label{#4}
    \end{subfigure}
    
    \caption{#5}
    \label{#6}
\end{figure}
}

\newcommand{\compareplotsvert}[6]{
\begin{figure}[H]
    \centering
    \begin{subfigure}{\textwidth}
        \centering
        \includegraphics[width=\linewidth]{#1}
        \caption{Calibrated Weak Model}
        \label{#3}
    \end{subfigure}
    \vspace{0.5cm}
    \begin{subfigure}{\textwidth}
        \centering
        \includegraphics[width=\linewidth]{#2}
        \caption{Uncalibrated Weak Model}
        \label{#4}
    \end{subfigure}
    
    \caption{#5}
    \label{#6}
\end{figure}
}

\newcommand{\compareplotsvertsmall}[6]{
\begin{figure}[H]
    \centering
    \begin{subfigure}{0.85\textwidth}
        \centering
        \includegraphics[width=\linewidth]{#1}
        \caption{Calibrated Weak Model}
        \label{#3}
    \end{subfigure}
    \begin{subfigure}{0.85\textwidth}
        \centering
        \includegraphics[width=\linewidth]{#2}
        \caption{Uncalibrated Weak Model}
        \label{#4}
    \end{subfigure}
    
    \caption{#5}
    \label{#6}
\end{figure}
}

\title{Flexible Routing via Uncertainty Decomposition}

\author{
  Charlotte Peale\textsuperscript{1,$\ast$,$\dagger$},
  Siddartha Devic\textsuperscript{2,$\dagger$} \\
  Parikshit Gopalan\textsuperscript{3},
  Udi Wieder\textsuperscript{3},
  Aravind Gollakota\textsuperscript{3} \\[1.5ex]
  \textsuperscript{1}Stanford University \quad
  \textsuperscript{2}University of Southern California \quad
  \textsuperscript{3}Apple
}

\begin{document}

\maketitle

{
  \renewcommand{\thefootnote}{\fnsymbol{footnote}}
\footnotetext[1]{Lead author. Correspondence to cpeale@stanford.edu.}
\footnotetext[2]{Work done while interning at Apple.}
}

\begin{abstract}
  A key strategy for balancing performance and cost in modern machine learning systems is to dynamically route queries to either a low-cost model or a more expensive oracle (such as a large pretrained model or human expert), an approach known as model routing. In this work we present a new uncertainty-aware router that (1) avoids unnecessary oracle calls on inherently ambiguous queries, and (2) adapts dynamically to different loss functions and cost parameters through simple hyperparameter changes, without retraining. Our method, applicable to any classification setting where multiple independent annotations per input are available, is based on decomposing total uncertainty into irreducible and reducible components using higher-order predictors \citep{ahdritzprovable}. This enables a unified approach to both routing and abstention: predict with the weak model when uncertainty is low, route to the oracle when reducible uncertainty is high, and abstain when irreducible uncertainty is high. Our router comes with strong theoretical guarantees bounding regret relative to optimal task-specific routers. We conduct experiments on both synthetic and real-world datasets that demonstrate the benefits of our approach in suitable regimes---in particular, whenever reducible and irreducible uncertainty are not too correlated.

\end{abstract}

\section{Introduction}
A core challenge in modern machine learning is navigating the trade-off between model performance and computational cost. As models have grown in scale and capability, they have also become more expensive to query in terms of cost, latency, and other resources. In order to build efficient and modular systems, it is increasingly important to be able to dynamically select the most appropriate model for a particular query or task, a task known as \emph{model routing}. Routing seeks to optimize this trade-off by handling ``simple'' queries with a low-cost model (e.g., a small local model or non-thinking LLM) and deferring ``complex'' queries to an expensive, high-quality oracle (e.g., a powerful server-side model, ChatGPT 5's ``Thinking'' mode, or human experts in automated data annotation pipelines).

There has been a large body of work on routing and closely related problems such as orchestration, selective prediction, and learning to defer. Most existing methods can be grouped into two main categories: \textbf{Confidence-based methods}, which leverage uncertainty estimates to route to the oracle (or alternatively abstaining) when the weak model's confidence is low (see e.g.~\cite{geifman2017selective}), and \textbf{Supervision-based methods}, which use an auxiliary router trained directly to predict when the stronger model will out-perform the weaker model (see e.g.~\cite{stronglearning, ongroutellm}). %
While both have their strengths, they also have key weaknesses. Confidence-based methods often fail to account for irreducible (aleatoric) uncertainty (such as ambiguous or noisy utterances), needlessly routing to the oracle even when it would also perform poorly. Conversely, supervision-based routers require expensive, task-specific training and lack the flexibility to adapt to changing API costs or downstream losses without retraining. See \cref{sec:additional-related-work} for an extended discussion of related work.

\paragraph{Our contribution}

We present a \emph{flexible, uncertainty-aware router} that integrates both routing and abstention into a unified framework. Our approach builds on the recently introduced abstraction of a higher-order predictor \citep{ahdritzprovable}, which predicts mixtures over label distributions.\footnote{We point out that higher-order predictors also arise naturally in the form of Bayesian models and ensembles, albeit potentially lacking the calibration guarantee.} Assuming a property known as higher-order calibration (see Definition~\ref{def:hoc}), we can decompose the expected loss on an input into irreducible (aleatoric) and reducible (epistemic) components. By leveraging this meaningful uncertainty decomposition, our method offers two key properties:
\begin{itemize}
    \item \textbf{Handles irreducible uncertainty:} It estimates a query's irreducible uncertainty to take appropriate action rather than needlessly escalating. The core routing strategy is straightforward: if the loss is primarily irreducible $\rightarrow$ abstain (or clarify with the user); if primarily reducible $\rightarrow$ route to the oracle; otherwise $\rightarrow$ predict using the weak model.
    \item \textbf{Zero-shot Flexibility:} it can dynamically adapt to different configurations of downstream loss, routing cost, and abstention penalties through simple hyperparameter adjustments, with no need for retraining.
\end{itemize}
\begin{figure}[t]
    \centering
    \begin{minipage}{0.53\textwidth}
        \centering
        \includegraphics[width=\textwidth]{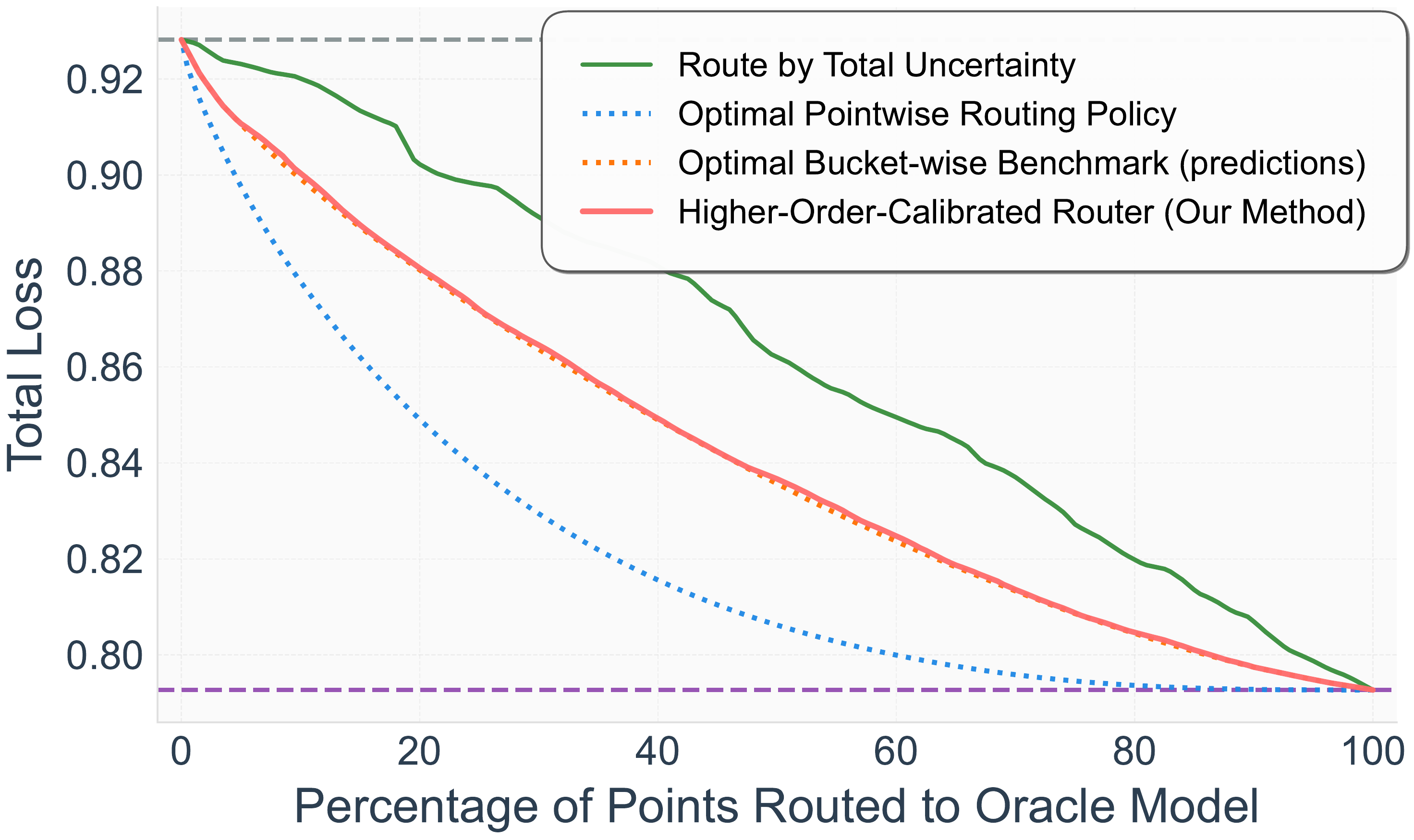}
        \caption{Routing performance of our method on synthetic data and a first-order-calibrated weak model compared to the point-wise optimal routing policy and a router based on total uncertainty. The data generating process is described in \cref{sec:synthetic_data_gen} and the ground-truth data is visualized in \cref{fig:johnson_fn}.}
        \label{fig:johnson_routing}
    \end{minipage}
    \hfill %
    \begin{minipage}{0.45\textwidth}
        \centering
        \resizebox{\textwidth}{!}{
\begin{tikzpicture}[
    scale=0.85, %
    transform shape, %
    region predict/.style={fill=blue!10},
    region route/.style={fill=green!10},
    region abstain/.style={fill=red!10},
    wavy border/.style={decorate, decoration={snake, amplitude=1pt, segment length=4pt}},
    thick dashed/.style={dashed, thick},
    axis line/.style={thick, ->, >=stealth},
    main text/.style={font=\sffamily\bfseries\small},
    label text/.style={font=\sffamily\footnotesize},
    math text/.style={font=\small}
]

    \def\xmax{6}      
    \def\ymax{4.5}      
    \def\alphaVal{1.8}  
    \def\betaVal{4.0}   
    \pgfmathsetmacro{\intersectX}{\betaVal - \alphaVal}

    \fill[region predict] (0,0) -- (\betaVal,0) -- (\intersectX,\alphaVal) -- (0,\alphaVal) -- cycle;

    \fill[region route] 
        (0,\alphaVal) -- (\intersectX,\alphaVal) -- (\intersectX,\ymax) 
        decorate [wavy border] { -- (0,\ymax) } 
        -- cycle;

    \fill[region abstain] 
        (\intersectX,\alphaVal) -- (\betaVal,0) -- (\xmax,0) 
        decorate [wavy border] { -- (\xmax,\ymax) -- (\intersectX,\ymax) } 
        -- cycle;

    \draw[thick dashed] (\intersectX, \alphaVal) -- (\betaVal, 0);
    \draw[thick dashed] (0, \alphaVal) -- (\intersectX, \alphaVal);
    \draw[thick dashed] (\intersectX, \alphaVal) -- (\intersectX, \ymax);

    \draw[axis line] (0,0) -- (\xmax + 0.5, 0);
    \node[anchor=north east, font=\sffamily\small, align=right] at (\xmax + 0.8, -0.2) 
        {\textbf{Irreducible Loss}\\ $L(f^*, f^*)$};
    \draw[axis line] (0,0) -- (0, \ymax + 0.5);

    \draw[thick] (0, \alphaVal) -- (-0.1, \alphaVal) node[left, math text] {$\alpha$};
    \draw[thick] (\betaVal, 0) -- (\betaVal, -0.1) node[below, math text] {$\beta$};
    
    \filldraw (\intersectX, \alphaVal) circle (1.5pt);
    \node[anchor=south west, font=\footnotesize, inner sep=2pt] at (\intersectX, \alphaVal) {$(\beta-\alpha, \alpha)$};

    \node[anchor=south, rotate=90, font=\sffamily\footnotesize, align=center] at (-0.5, \ymax/2) 
        {\textbf{Reducible Loss}\\$L(f^*,f) - L(f^*,f^*)$};

    \node[main text, text=blue!60!black] at (0.35*\betaVal, 0.45*\alphaVal) {PREDICT};
    \node[main text, text=green!60!black] at (0.5*\intersectX, \alphaVal + 1.0) {ROUTE};
    \node[main text, text=red!60!black] at (\intersectX + 1.8, \alphaVal + 1.0) {ABSTAIN};

\end{tikzpicture}
}
\caption{Optimal routing decisions for fixed penalties $\alpha$ and $\beta$.}
\label{fig:optimal_router_landscape}
    \end{minipage}
\end{figure}

\paragraph{Theoretical properties} Our routing framework comes with theoretical guarantees that bound the regret relative to using an \emph{optimal router} that perfectly knows both the weak model's and the oracle's predictions at every point. The strongest form of our guarantee, stated informally, is as follows. Suppose we have a weak model $\pred$ that we are able to higher-order calibrate with respect to an oracle $\pbayes$, which we assume to be the true conditional distribution for our main results, but later generalize. At a high-level, the higher-order calibration guarantee implies that for every weak model prediction $v$, we have access to the true mixture of $\pbayes(z)$ ranging over all $z$ such that $\pred(z) = v$. %
Then, we can instantiate an uncertainty-aware router that, simultaneously for all task configurations (namely downstream loss, abstention cost, and routing cost), performs at least as well as any task-specific router that is a function of the weak model $\pred$. One interpretation of our results is that they provide omniprediction-style guarantees \citep{gopalan2022loss} for routing. We provide the full technical details in \cref{sec:routing_from_hoc}.%

\paragraph{Experiments} We run experiments on a variety of real-world and synthetic datasets to obtain a nuanced understanding of the benefits of our method. In \cref{fig:johnson_routing} we show how our method yields strong routing performance when applied to a synthetic binary classification dataset. However we also observe that in certain scenarios, our method  does not yield significant improvements over simpler baselines. To explain this, we lay out different regimes of uncertainty-aware routing, and describe when routing based on reducible uncertainty yields the most benefit. Briefly, we expect to see a benefit when reducible uncertainty is not too correlated with irreducible uncertainty. Experiments are detailed in \cref{sec:experiments}.

\section{Related Work}\label{sec:additional-related-work}
Model routing and related problems such as orchestration, learning to defer, and selective prediction have been the subject of a vast body of work. Here we do not attempt a complete survey but instead aim to highlight the general principles of these lines of work and how they relate to ours. We note at the outset that our method is not intended to be fully general-purpose but rather assumes: (a) a binary or multiclass classification setting with access to a calibration dataset consisting of multiple labels per input, (b) an oracle that is strictly stronger than (not complementary to) the weak model. A significant difference from all prior work is that to our knowledge, ours is the first method that performs both abstention and routing in a unified way, and with strong distribution-free theoretical guarantees. %

\textbf{Selective prediction.} The idea of augmenting machine learning models with an option to abstain (or reject) dates back at least to \cite{chow1957optimum}; see \cite{hendrickx2024machine} for a modern survey. The cost of abstaining is usually modeled as a fixed cost (as we do), with the goal of a selective predictor being to minimize this modified loss function. Since it is typically nonconvex, many works have studied minimizing suitable surrogates instead (see e.g. \cite{cortes2016learning,mao2024theoretically}). The decision to abstain is made either based on the model's own confidence or a separately supervised ``rejector'' model. Recent literature in this area has proposed evaluating selective predictors in terms of their risk-coverage curves, and various works have studied optimal strategies and inherent tradeoffs under certain idealized settings \citep{el2010foundations,geifman2017selective}. In the abstention-only setting our router specializes to a confidence-based policy that is justified by a higher-order form of calibration.

\textbf{Learning to defer.} Here the goal is precisely to learn to defer to an oracle on inputs where there is a benefit to doing so; see \cite{stronglearning} for a survey. Routing to the oracle is modeled as incurring an instance-dependent cost which may be a function of oracle's loss, but may in general be arbitrary. For this reason learning to defer is often viewed as a generalization of selective prediction where the abstain option carries an adaptive or instance-dependent penalty rather than a fixed one. The focus of many papers in this area is on tractable approaches for minimizing the resulting modified cost function, or surrogates thereof, usually over a fixed model class (see e.g.~\cite{awasthi2022h,mohri2023learning,mao2023two}). In our setting we effectively consider the special case where the oracle has a known form, and the cost incurred upon routing is the oracle's loss plus a fixed routing penalty.

\textbf{Model cascades.} In this approach we run a suite of models in a cascade ordered by increasing cost, passing the output each time to an auxiliary verifier or quality estimation model, and stop when the output meets a desired threshold (see e.g. \cite{enomoro2021learning}). In this sense cascading may be considered a sequential learning to defer problem. %

\textbf{Orchestration.} Unlike our setting, orchestration considers a suite of models with complementary strengths, and the goal is to learn to route to the most suitable model for a given query. As there is no relationship between the models, this approach largely rests on using supervision to train a router. The Mixture of Experts (MoE) architecture can be seen as a version of this paradigm where the models (or experts) and router (or gating function) are trained in tandem. In some variants, routing is also done based on clustering queries in embedding space. See e.g.~\cite{dekoninck2024unified,jitkrittum2025universal} and references therein. %

\textbf{Loss Prediction, Omniprediction} The ability of our router to dynamically adapt to new losses and routing penalty draws on work on omnipredictors~\citep{gopalan2022omnipredictors}, which presents a way to learn a single predictor that can be post-processed to minimize any loss from a rich class of losses. Follow-up work has more explicitly made the connection between the ability to accurately estimate a model's loss and various calibration conditions \cite{gopalan2022loss, gollakota2025does, okoroafor2025nearoptimal}. Omniprediction has also been extended to the selective abstention setting by \cite{casacuberta2025selective}. Our work can be viewed as a higher-order extension of these ideas, allowing us not only to estimate the overall loss, but decompose it into irreducible and reducible components in a way that enables more principled routing decisions.

\textbf{Complementarity, Human-AI Collaboration}
Our work is also related to a line of theoretical results from Human-AI collaboration.
Works such as \citet{raghu2019algorithmic,donahue2022human, alur2024human, alur2024integrating} study ways of \emph{combining} human and algorithmic predictions.
Importantly, this requires understanding when and where predictions from different sources are \emph{complementary} and can be combined to improve overall decision quality.
In contrast, our work considers routing to an oracle model which we imagine to be strictly more powerful than the cheaper model everywhere.

\textbf{Routing in LLMs} 
There are a number of papers related to routing in the LLM literature \citep{dinghybrid, ding2025best,ongroutellm}. 
These works usually consider the problem of routing a user query to one of a number of different LLMs which may differ on quality, cost per token, or domain expertise.
The overall goal of LLM routing is to maximize performance subject to a given budget either per query or aggregated over the entire dataset.
\citet{jitkrittum2025universal} propose a router for an arbitrary set of LLMs which embeds both LLMs and prompts / queries in a particular feature space to make routing decisions.
Their approach treats the possible LLMs as a \emph{dynamic pool}, which then allows the router to still function without re-training when certain LLMs are deprecated, or new LLMs are added to the pool. \citet{somerstep2025carrot} point out that any optimal LLM routing approach should take into account both the predictive power of each LLM, as well as their cost per token.
They use this insight to develop algorithms and lower bounds for LLM routing based on predicting both cost and accuracy.
For a comprehensive survey on LLM routing, we refer to \citep{varangot2025doing}.
More generally, in contrast to our approach, most LLM routing papers rarely explicitly consider epistemic and aleatoric uncertainty decompositions, and simply optimize the pareto-frontier of certain performance-cost curves.

\section{Problem Formulation}
\subsection{The Routing Framework} We consider a classification setting with instance space $\calX$ and discrete label space $\calY$, with the conditional distribution of $y | x$ denoted by $\pbayes(x) \in \DY$ (the Bayes-optimal predictor). Here $\DY$ denotes the space of distributions over $\calY$. We also assume access to a downstream target distribution $\calD_{\calX}$ that inputs $x$ follow during deployment. This distribution is primarily for calibration purposes, and may differ from the original training distribution for the weak model.

We consider two predictive sources: a low cost, potentially error-prone \textbf{weak model} $\pred: \calX \rightarrow \DY$, and a high-cost \textbf{oracle} $\pbayes: \calX \rightarrow \DY$ providing the true conditional distribution\footnote{We note that our framework and results can be extended to a more general pool of oracles whose expected loss can be computed as a function of $\pbayes$, a property we term \emph{Bayes-dependency}. See discussion in Appendix~\ref{app:bayes-dep-oracles} for details.}.
For any input $x$, the router must either \textbf{predict} using the weak model, \textbf{route} the query to the oracle, or \textbf{abstain} from making a prediction altogether (and potentially take fallback actions such as requesting clarification, flagging for later inspection, etc.).

The cost of these actions is determined by a \emph{routing configuration} $T = (L, \alpha, \beta)$, which specifies a bounded loss function $L: \DY \times \DY \rightarrow [0, B]$, a routing penalty $\alpha \geq 0$, and abstention penalty $\beta \geq 0$. Here $L\divx{\distr^*}{\distr} := \ex_{y \sim \distr^*}[\ell(y, \distr)]$ denotes the expected loss of prediction $\distr \in \DY$ on a point with ground truth distribution $\distr^* \in \DY$. The total cost of each action is a combination of computational overhead and prediction error:
\begin{itemize}
    \item \textbf{Predict:} Using the weak model incurs no computational overhead. The cost is determined solely by the expected prediction loss, $L\divx{\pbayes(x)}{ \pred(x)}$.
    \item\begin{sloppypar} \textbf{Route:} Querying the oracle yields the optimal prediction $\pbayes(x)$, which incurs loss $L\divx{\pbayes(x)}{\pbayes(x)}$. However, this incurs a fixed computational overhead of $\alpha \geq 0$, representing the additional time or resources required to query the stronger model.\end{sloppypar}
    \item \textbf{Abstain:} Declining to predict incurs no computational overhead but results in a fixed penalty of $\beta \geq 0$, representing the cost of failing to provide a response.
\end{itemize}

The total cost for an action $a \in \{P, R, A\}$ (representing ``Predict,'' ``Route,'' and ``Abstain'') given an $x \in \calX$ and configuration $T = (L, \alpha, \beta)$ is thus:

\begin{align*}
    \mathsf{Cost}(x, a; T) := 
    \begin{cases}
        L\divx{\pbayes(x)}{\pred(x)} & a = P\\
        L\divx{\pbayes(x)}{\pbayes(x)} + \alpha & a = R\\
        \beta & a = A
    \end{cases}
\end{align*}

\subsection{Optimal Routing and Loss Decomposition}\label{sec:optimal-routing}
The \emph{pointwise-optimal} routing policy $r^*$ for a configuration $T = (L, \alpha, \beta)$ is the policy that selects the cost-minimizing action for every $x \in \calX$:
\[r^*(x, T) := {\arg\min}_{a \in \{P, R, A\}}\mathsf{Cost}(x, a; T).\]
We now derive the exact form of the optimal router by analyzing the relationship between the weak model's error and the oracle's error. For simplicity, we focus on the case where $L$ is a \emph{proper loss}, meaning that predicting the true conditional distribution minimizes the expected loss, but the general case can be easily reduced to the proper case (see \cref{sec:loss-explanation} for details).

Given an input $x$ with true conditional distribution $\pbayes(x)$, the weak model predicts $\pred(x)$ while the oracle predicts $\pbayes(x)$. Thus the weak model's loss is $L\divx{\pbayes(x)}{\pred(x)}$, while the oracle's loss is $L\divx{\pbayes(x)}{\pbayes(x)}$. For a proper loss, the latter is no greater than the former, and so we may decompose the weak model's loss into two key components:
\begin{enumerate}
    \item The \textbf{Irreducible Loss} $L\divx{\pbayes(x)}{\pbayes(x)}$: This captures the loss incurred due to inherent noise in the label distribution that even the oracle cannot eliminate.\footnote{In fact, $L\divx{\pbayes(x)}{\pbayes(x)}$ is precisely the \emph{entropy} $H(\pbayes(x))$ associated with the proper loss $L$. For example, when $L$ is the cross-entropy loss, $H$ is the Shannon entropy.}
    \item The \textbf{Reducible Loss} $L\divx{\pbayes(x)}{\pred(x)} - L\divx{\pbayes(x)}{\pbayes(x)}$: This represents the excess loss due to the weak-predictor's sub-optimality, and is exactly the expected difference in loss incurred by trusting the weak model's predictions vs.\ querying the oracle.
\end{enumerate}

The amount of irreducible loss and reducible loss precisely characterizes which routing action is optimal. Simplifying the cost minimization expression, we can characterize the decisions of the optimal policy in terms of the reducible ($\RL$) and irreducible ($\IL$) loss.

This is illustrated in Figure~\ref{fig:optimal_router_landscape}. See \cref{app:optimal-routing-decisions} for an explanation on how to derive these optimal actions in terms of irreducible and reducible loss.  %

The optimal policy demonstrates that routing is ineffective for points with high irreducible loss. This explains why naive confidence-based routers often route wastefully; in such cases, the optimal action is typically to abstain or simply trust the weak model.

\section{Prediction-Optimal Routing via Higher-Order Calibration}\label{sec:routing_from_hoc}
We have seen that the optimal routing decision depends on the \emph{reducible loss} at every point, namely the gap between the weak model's error and the irreducible or aleatoric error. Of course, computing this gap exactly at each point requires access to the true $\pbayes$, which we do not have without making an oracle call in the first place. In this section we define a more realistic benchmark for performance: the \emph{prediction-only optimal router}, which represents the best possible router whose decisions are a function of the weak model. We then show that \emph{higher-order calibration} \citep{ahdritzprovable} provides exactly the necessary uncertainty information to construct a flexible router that competes with the optimal prediction-only policy.

\subsection{Prediction-Only Routing}
The prediction-only optimal router, $r_f: \calX \times \calT \rightarrow \{P, R, A\}$, is the router that minimizes the routing cost while being restricted to taking the same action on all points given a certain prediction by $\pred$ (the ``level sets'' of $\pred$). The optimal action for a particular prediction is the one that minimizes the routing cost on average over the level set:

    \begin{align*} &r_{\pred}(x, T) := {\arg\min}_{a \in \{P, R, A\}} \ex_{x' \sim \calD_{\calX}}[\mathsf{Cost}(x', a; T) | \pred(x') = \pred(x)].
    \end{align*}

Our goal is to build a flexible router that achieves performance comparable to this relaxed benchmark.

\subsection{Why Higher-Order Calibration is Helpful}

In contrast to pointwise-optimal routing, determining the prediction-only optimal action only requires good estimates of the reducible loss \emph{on average over each level set}. This requirement is reminiscent of a calibration guarantee. As a reminder, $\pred$ is said to be calibrated if the average true label matches its predictions over every level set, e.g. 
\[\distr = \ex_{x \sim \calD_{\calX}, y \sim \pbayes(x)}[y | \pred(x) = \distr].\]
Existing work has shown that calibrated predictors can be used to construct good estimates of the average loss incurred over each level set~\citep{gollakota2025does, gopalan2022loss}. However, estimates of the total loss alone are not enough for our purposes --- we need accurate estimates of how the total loss is decomposed into its irreducible and reducible components. To disentangle them, we need the stronger notion of \emph{higher-order calibration}, introduced by~\citet{ahdritzprovable}.

Higher-order calibration is a property of a \emph{higher-order predictor} $\mixpred: \calX \rightarrow \DDY$. Unlike standard predictors that output a single label distribution, a higher-order predictor outputs a \emph{distribution over distributions}, quantifying uncertainty about the true $\pbayes(x)$ rather than just the label $y$.
Higher-order calibration is a form of calibration that guarantees that such higher-order predictions are meaningfully tied to the target distribution: 

\begin{definition}[Higher-Order Calibration, \cite{ahdritzprovable}]\label{def:hoc}
    A higher-order predictor $\mixpred: \calX \rightarrow \DDY$ is $\epsilon$-\emph{higher-order calibrated} if for every value $\mix \in \DDY$ in the range of $\mixpred$, the predicted mixture $\mix$ is close to the \emph{Bayes-mixture} $\mixbayes$ corresponding to the distribution of $\pbayes(x)$ for $x \sim \calD_{\calX}$, conditioned on $\mixpred(x) = \mix$. Formally, for every $\mix \in \text{range}(\mixpred)$, $\epsilon$-higher-order-calibration requires that $W_1(\mix, \mixbayes) \leq \epsilon$. Here, $W_1$ denotes the Wasserstein distance with respect to the $\ell_1$ norm.
\end{definition}

\subsection{Main Result: Higher-Order Calibration Enables Optimal Routing}
We now show that higher-order calibration is exactly the property needed to construct accurate estimates of the irreducible and reducible loss, on average over level sets. These estimates enable us to make near-optimal routing decisions compared to the prediction-only routing policy. 

We make the simplifying assumption that the higher-order predictor, $\mixpred$, shares the same level sets as the weak predictor, $\pred$. This means $\pred$ yields constant predictions across any level set of $\mixpred$. Such a condition arises naturally if $\mixpred$ is obtained through post-hoc calibration of the level sets of $\pred$. We will explore ways to relax this assumption later in the text. With this assumption in hand, we present our main result:

\begin{restatable}[Higher-Order Calibration Implies Low Prediction-Only Routing Regret]{thm}{hocimpliesrouting}
    \label{thm:hoc_implies_routing}
    Suppose we have an $\epsilon$-higher-order calibrated predictor $\mixpred$ whose level sets match those of $\pred$. Then we can construct a flexible router $r$ based on $\mixpred$ with the following near-optimality property. For any routing configuration $T = (L, \alpha, \beta)$ where $L$ is a $B$-bounded proper loss and $\alpha, \beta \geq 0$, we have \[\ex_{x \sim \calD_{\calX}}[\mathsf{Cost}(x, r(x, T); T) - \mathsf{Cost}(x, r_{\pred}(x, T); T)] \leq B\epsilon.\]
\end{restatable}

The proof of this result can be found in Appendix ~\ref{sec:hoc_implies_routing_pf}. Note that we do not require anything about the predictions of the weak model, such as calibration.

\section{Practical Higher-Order Predictors}\label{sec:practical-discussion}
In \cref{sec:routing_from_hoc}, we established higher-order-calibrated predictors as a key tool that can be leveraged to build a flexible routing system that can adapt to different configurations $T = (L, \alpha, \beta)$. In this section, we discuss how to construct such a higher-order calibrated predictor in a practical manner. This requires addressing two engineering challenges: (1) collecting a suitable calibration post-processing dataset, and (2) designing an efficient partition/binning of the input space. Our approach is based on methods for post-hoc higher-order calibration proposed by \citet{ahdritzprovable}. We refer the reader to their paper for further discussion of alternative higher-order calibration methods. 

\subsection{Basic Higher-Order Calibration Approach}
To estimate the mixture of conditional expectations $\pbayes(x)$ within a level set, we cannot rely on standard single-label validation data. 
We instead require a \emph{$k$-snapshot calibration set}: a small dataset where each input $x$ is paired with \emph{multiple} independent labels sampled as $y_1, \dots, y_k \sim \pbayes(x)$ (e.g., multiple expert annotations). The mean of a $k$-snapshot, which we denote $\overline{y} := \frac{1}{k}\sum_{i = 1}^k y_i \in \DY$ (where the labels $y_i$ are viewed as one-hot vectors in $\DY$), serves as an estimate of the true conditional distribution $\pbayes(x)$. The core of our method lies in using the \emph{empirical mixture} over such $\overline{y}$ as an estimate of the true mixture of $\pbayes(x)$ over a particular bin.

Specifically, our post-hoc higher-order calibration method uses a ``Bin-and-Estimate'' procedure:
\begin{enumerate}
\item \textbf{Partition:} Define a partition $\calP$ of the input space $\calX$. These are the ``bins'' (or level sets) on which we will make decisions. 
\item \textbf{Estimate:} For each bin $P \in \calP$, save the empirical distribution (aka normalized histogram) of (prediction, snapshot mean) tuples $(\pred(x), \overline{y})$ falling within that bin. We denote this empirical distribution $S_P$. Note that the empirical distribution of $\overline{y}$ is precisely a mixture distribution; what we save as $S_P$ can be thought of as a \emph{tagged} mixture where the mixture components are tagged with an $f(x)$ value.

\item \textbf{(Optional Re-Calibration):} Update the values of $\pred$ to be the centroid of the true labels in its corresponding bin. For $x \in P$,
\[\pred(x) := \frac{1}{|S_P|}\sum_{(\pred(x'),\overline{y}) \in S_P}\overline{y}.\]
Make sure to update the saved predictions in $S_P$ as well. 

\item \textbf{Route:} For a new query $x$ and for any given routing configuration $T = (L, \alpha, \beta)$, find the bin $P$ containing $x$. Then use the saved empirical distribution $S_P$ to estimate the expected reducible and irreducible loss, and make the appropriate decision to predict, route, or abstain:
\[\hat{\IL} := \frac{1}{|S_P|}\sum_{(\pred(x'), \overline{y}) \in S_P}L(\overline{y},\overline{y}) 
\quad\quad\quad \hat{\RL} := \frac{1}{|S_P|}\sum_{(\pred(x'), \overline{y}) \in S_P}L(\overline{y}, \pred(x')) - \hat{IL}.\]

\[r(x, T) = {\arg\min}_{a \in \{P, R, A\}}\begin{cases}
    \hat{\IL} + \hat{\RL} & a = P\\
    \hat{\IL} + \alpha & a = R\\
    \beta & a = A
\end{cases}\]
\end{enumerate}

Notice that the statistics saved (namely $S_P$ for each bin $P$) are completely independent of the routing configuration, which enters only in step 4. This is the source of the flexibility of our approach: for any given configuration, step 4 can be carried out cheaply using the saved statistics.

\begin{remark}
While the empirical mean $\frac{1}{k}\sum_{i = 1}^k y_i$ concentrates around $\pbayes(x)$ for large $k$, using it as a proxy when $k$ is small yields biased loss estimates. In such cases, we recommend the unbiased estimation techniques from \cite{ahdritzprovable} (e.g. Theorem 4.3) to accurately decompose the irreducible and reducible loss. For our experiments we work with the empirical mean for simplicity. 
\end{remark}

\subsection{Selecting a Partition}
The success of this framework hinges on the choice of partition $\calP$, which governs the tradeoff between granularity and estimation stability. While fine-grained partitions (e.g. exact level sets) allow for precise routing decisions, they often suffer from data sparsity, meaning that there may not be sufficient data to have good estimates of the loss decomposition in each partition (especially in multiclass prediction, since the number of level sets scales exponentially with the number of classes). Coarser partitions ensure each bin has enough samples for a stable estimate, but they force the router to apply the same policy to broad regions of the input space.  

Even restricting to partitions that are coarse enough to have enough samples for each cell, many ways to divide the space remain. At a high level, a good partition should try to \emph{minimize the variance of the reducible and irreducible loss} within each cell. We offer a formalization of this intuition in Appendix~\ref{app:choosing-good-partitions} along with a discussion of recommended strategies for selecting buckets.

\section{Experiments}\label{sec:experiments}

In this section, we implement our routing framework following the method described in Section~\ref{sec:practical-discussion}, and study its performance on a number of synthetic and real-world datasets. 

Direct comparison to existing baselines is nuanced, as our method targets a broader scope with three primary advantages. First, unlike confidence-based methods that consider only total uncertainty, our framework routes effectively under irreducible uncertainty, avoiding unnecessary oracle queries. Second, it offers flexibility, providing optimal routing for any proper loss without retraining, whereas supervised methods are fixed during training. Finally, it supports a richer action space---predicting, routing, or abstaining---while existing frameworks are generally limited to two options and are difficult to extend.

To demonstrate these different benefits of our routers, we divide our experiments into three parts:
\begin{enumerate}
    \item \textbf{Routing with Irreducible Uncertainty}: In this section, we demonstrate our router's ability to distinguish between hard points with high reducible loss that should be routed, vs ambiguous points with high irreducible loss that will not benefit from routing, a distinction that confidence-based methods fail to make. For this section, we focus on the simplified routing setting where the available actions are only to route or predict. 
    \item \textbf{Loss Flexibility}: In this section, we study the brittleness of supervised methods in settings with dynamically changing losses. We show that our method can adapt without the need for retraining to a variety of loss functions, while supervised methods degrade on losses that differ from the fixed configuration used during training. 
    \item \textbf{Three-way Routing Decisions}: In this section, we show that our method can gracefully expand beyond 2-way decisions (route vs. predict) to accommodate a third option: abstention. 
\end{enumerate}

\subsection{Experiment Setup}

\subsubsection{Datasets and Model Architectures}
We evaluate our method on a synthetic binary prediction task and three real-world datasets characterized by inherent human label ambiguity: CIFAR-10H, SNLI, and ChaosNLI. 
For each dataset, we train a ``weak predictor'' $\pred$ whose prediction we will use when deciding whether to route to an oracle or not. See \cref{sec:dataset-details} for detailed information on each dataset and its associated weak predictor.

\paragraph{Oracle Predictor ($\pred^*$)} For each dataset and tested loss, we compute the loss of the oracle predictor using the label distribution defined per dataset.
For example, if we are testing CIFAR-10H with the cross entropy loss, the oracle loss would be the entropy of the label distribution of the 47 human annotations on the particular input.
We have intentionally chosen datasets with varying amounts of intrinsic aleatoric uncertainty. Therefore, the oracle loss will usually be larger than zero, implying some disagreement exists between human labelers.

\paragraph{Routing Framework} We implement our routing framework using the post-hoc ``Bin-and-Estimate'' procedure detailed in Section~\ref{sec:practical-discussion}. To ensure sample efficiency, we employ a top-class binning approach, defining partitions based on the predicted class and its corresponding confidence score. For each class, we divide these probabilities into 10 quantile-based buckets, ensuring an equal distribution of calibration points across bins.

\subsubsection{Baselines and Evaluation}

\paragraph{Two-way Routing Curves}
While our method naturally handles three-way decisions (predict, route, or abstain), existing frameworks such as \citet{hendrickx2024machine,el2010foundations} are typically restricted to binary actions (predict/route or predict/abstain). To ensure a fair comparison, our first two experiments focus on these two-way settings using cost-agnostic routing curves, which plot the percentage of points routed ($x$-axis) against total loss ($y$-axis).

This evaluation is deliberately favorable to baselines: by sweeping across all possible thresholds, the curve effectively grants every method its optimal post-hoc operating point. While our framework is designed to explicitly optimize for specific penalties $\alpha$ and $\beta$, the curve treats these costs as unknown, selecting the loss-minimizing threshold for every point on the plot. Any strategy that orders points by priority can generate such a curve; superior strategies yield sharply declining curves, while random selection results in a linear decrease. 

Lastly, in our third experiment, we demonstrate that our method seamlessly adapts to three-way decisions and variable costs (scenarios where simple thresholding often fails).

\paragraph{Routing Strategies and Baselines}

We compare our method’s routing curves against several established strategies and optimal benchmarks:
\begin{itemize}
    \item \textbf{Route by Total Uncertainty}: Points are ranked by their predictive entropy $H(\pred(x)) = L(\pred(x), \pred(x))$, where $H$ corresponds to the target loss. Points with high entropy are prioritized for routing.
    \item \textbf{Supervised Router}: To evaluate our framework against a direct supervised alternative, we designed a regression-based router specifically for these tasks. This model --- an XGBoost Regressor --- is trained on the calibration dataset to predict the amount of reducible loss at each point, using the prediction $\pred(x)$ and input features $x$ to estimate the gap $\ell(y, \pred(x)) - \ell(y, \pbayes(x))$. While this represents a natural supervised approach to our routing objective, it is inherently limited in two ways. First, while it performs well for binary decisions (e.g., route vs. predict)—often matching or even outperforming our method on synthetic data—it cannot be directly extended to multi-way routing (e.g., adding an "abstain" option). Second, the regressor is tied to the specific loss used during training, whereas our method remains flexible across a variety of losses without retraining (see Benefit 2).
    \item \textbf{Point-wise Optimal Router}: An optimal baseline that orders points based on the true reducible loss, $L(f^* (x), f(x)) - L(f^* (x), f^* (x))$.
    \item \textbf{Bucket-Optimal Router} (occasionally referred to as \textbf{Prediction-Optimal Router} for simplicity): A relaxed baseline representing the best possible routing decisions that remain constant over the buckets used in our framework. This router measures the average reducible loss within each bucket on the test set and orders points accordingly.
\end{itemize}

\subsection{Benefit 1: Routing in the Presence of Irreducible Uncertainty}
For simplicity, we first consider a binary routing setting: the system must either predict using the base model or route to the oracle, which in this case predicts the true ground-truth conditional probability. In this simplified setting, the optimal router routes any point with reducible loss that is above a particular threshold. As discussed in \cref{sec:optimal-routing}, routing strategies that fail to decompose the amount of reducible vs irreducible loss may fail in settings where the two quantities are de-coupled. 

We demonstrate this advantage in \cref{fig:johnson_routing}. In this synthetic environment, routing based on total uncertainty of even a well-calibrated predictor is significantly outperformed by our higher-order method. By explicitly estimating and routing based on the reducible loss, our framework avoids ``wasting'' the routing budget on inherently ambiguous points that the oracle cannot improve upon.

\paragraph{When Decomposition Matters} The practical benefit of separating irreducible and reducible loss is intrinsically tied to the joint distribution of these two quantities. We note that both the structure of the data and weak predictions as well as the bucket choice for binning-based methods such as ours can greatly affect to what degree decomposition is helpful for routing. See \cref{sec:decomp-discussion} for further discussion of these effects.

\paragraph{Observations on Real-World Datasets} We extended this analysis to several real-world datasets characterized by human label ambiguity: CIFAR-10H, SNLI, and ChaosNLI. Interestingly, we observed that in these datasets, total and irreducible loss remain highly correlated across the level sets of standard models. Consequently, decomposition did not yield the same dramatic performance gaps seen in our synthetic experiments.

This correlation persists in spite of non-trivial variation in the amount of irreducible loss in these datasets, suggesting that it is likely a product of either the training process or our bucketing approach rather than the data itself. We note that this finding aligns with prior observations by \cite{mucsanyi2024benchmarking}, who noted a high empirical correlation between aleatoric and epistemic uncertainties in trained models. 

We provide examples of this behavior in \cref{fig:routing_cifar10_recal,fig:routing_snli_recal,fig:routing_chaosnli_recal} on the CIFAR-10H, SNLI, and ChaosNLI datasets, respectively. When the base model is well-calibrated, most methods converge toward the bucket-optimal routing curve. This suggests that in some settings routing by total uncertainty may be a sufficient proxy for estimates of reducible loss. However, our synthetic demonstrations highlight that determining whether decomposition is important should be a key first step when implementing a router. 

\begin{remark}
    We observed that many of the models we trained were significantly miscalibrated when used out of the box, meaning that our method performed significantly better than the total uncertainty baseline. This highlights that our method can continue to perform well even in settings where the model is miscalibrated. However, the fact that we observe these benefits collapsing after calibrating the model suggests that most of the observed improvement comes from miscalibration, rather than a failure to decompose irreducible and reducible loss. 
\end{remark}

\subsection{Benefit 2: Loss Flexibility}

A key benefit of our framework is its flexibility. In particular, it can quickly adapt to new routing penalties, loss functions, and even oracles. In this section, we focus on its flexibility with respect to loss functions. We compare to a natural baseline -- a supervised ``loss predictor,'' which takes as input a prediction and features about an input, and outputs a prediction of the amount of reducible loss. 

In a static setting where the loss is fixed, this supervised approach tends to perform very well, even out-performing our partition-based method on some datasets. However, in dynamic settings where the loss we care about may change, the downside of this approach is that we must retrain a new loss predictor for each new loss. \cref{fig:loss_comp_snli_no_pp} depicts the potential loss in performance resulting from using a single statically trained loss predictor to make routing decisions on new losses (detailed in \cref{sec:loss-explanation}) on the SNLI dataset. In contrast, our method preserves performance guarantees across a variety of different loss functions. We present similar plots for all datasets in \cref{sec:loss-comp-plots}.

\begin{figure}
    \centering
    \includegraphics[width=\linewidth]{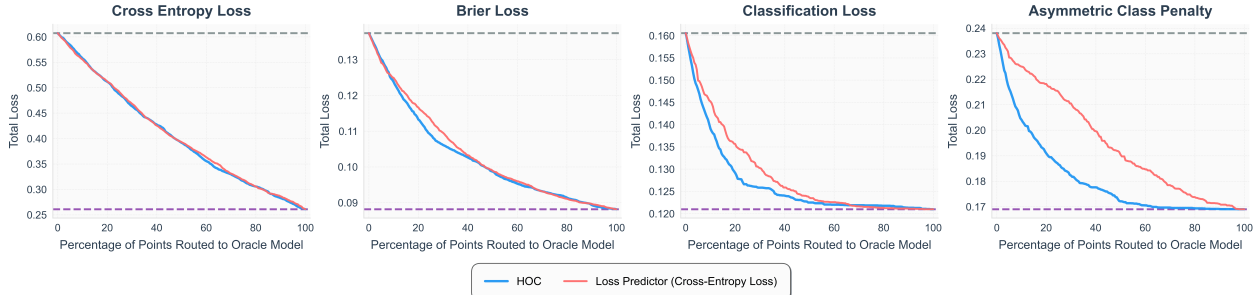}
    \caption{Routing performances for different target loss functions for the SNLI dataset. The loss predictor baseline is a static loss predictor trained to predict the amount of reducible cross-entropy loss. While our method (HOC) performs roughly comparably to the loss predictor when routing based on cross-entropy loss, we see the performance of the loss predictor degrade when deployed on new losses further from the cross-entropy loss, whereas our method maintains good performance.}
    \label{fig:loss_comp_snli_no_pp}
\end{figure}

\subsection{Benefit 3: Effective Three-Way Decisions}


Thus far, we have focused on two-way decisions between prediction and routing. However, a key benefit of our method is that it can gracefully adapt to multi-way decisions beyond simply choosing whether to route or predict. In particular, in this section we demonstrate its ability to decide not only between prediction and routing, but also whether to abstain when the irreducible loss is too high, where the optimal decision for this setting is visualized in \cref{fig:optimal_router_landscape}. 

\begin{figure}[ht]
    \centering
    \includegraphics[width=0.6\linewidth]{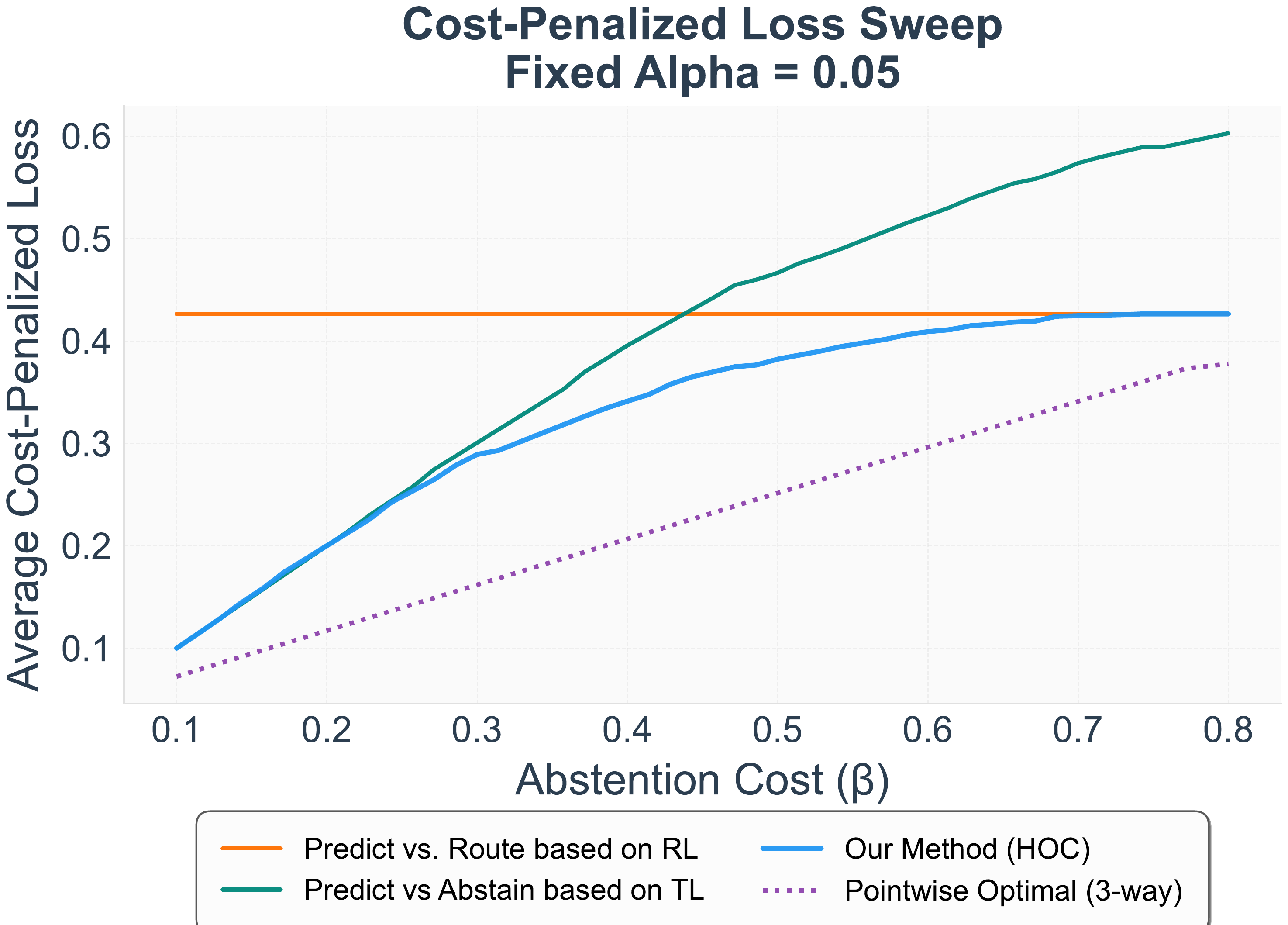}
    \caption{Cost-penalized performance of our method on the three-way decision task (Predict vs. Route vs. Abstain) on the SNLI dataset, for a fixed routing penalty of $\alpha = 0.05$ while sweeping the abstention penalty $\beta$. Our method gracefully adapts to the three-way decision setting, out-performing the two-way decision baselines of predict/route (orange) and predict/abstain (green).}
    \label{fig:snli_three_way}
\end{figure}

\cref{fig:snli_three_way} depicts the cost-penalized loss of our method on the SNLI dataset, with similar plots for other datasets included in \cref{sec:three-way-plots}. In this plot, the routing cost ($\alpha$) is kept fixed and the abstention cost ($\beta$) is swept from 0.1 to 0.8. We compare the full three-way power of our method against restricted two-way strategies (predict vs. route, and predict vs. abstain) using our calibrated estimates of the irreducible and total loss, respectively. As shown, the ability of our method to extend to three-way decisions allows it to achieve lower loss than a simpler framework that allows for only two options, effectively tracing the lower envelope of the two-way decision curves.

\section{Conclusions}
This work demonstrates the need for theoretically rigorous routing frameworks capable of adapting to unanticipated downstream desiderata such as new loss functions, costs, or routing oracles.  We highlighted the importance of uncertainty decomposition for optimal routing, and proposed a practical and provably effective routing mechanism based on higher-order calibrated predictors \citep{ahdritzprovable}, validated by empirical experiments. We view developing more efficient methods for learning high-quality higher-order-calibrated predictors as a key direction of future work that could further enhance the scalability of our framework.

\section*{Acknowledgments}
CP is supported by the Simons Foundation Collaboration on the Theory of Algorithmic Fairness and the Apple Scholars in AI/ML PhD fellowship. 

\bibliographystyle{plainnat}
\bibliography{ref}

\medskip

\appendix

\section{Useful Results about Losses}
In this section, we prove a few useful results about bounded proper losses that will be useful in the proofs used throughout the paper. Given a loss $\ell: \calY \times \DY \rightarrow [0, B]$, we denote its expectation over a conditional distribution as $L\divx{\distr^*}{\distr} = \ex_{y \sim \distr^*}[\ell(y, \distr)].$

We say that a function $g: \DY \rightarrow \R$ satisfies $L$-Lipschitz continuity if for every $\distr_1, \distr_2 \in \DY$, 
\[g(\distr_1) - g(\distr_2) \leq L\|\distr_1 - \distr_2\|_1.\]

The first lemma shows that any bounded proper loss satisfies Lipschitz continuity in its first argument. 
\begin{lemma}\label{lem:lipschitz-pstar}
    Let $L$ be a proper loss with range $[0, B]$. For any $\distr_1, \distr_2, q \in \DY$, $L$ satisfies
    \[\left|L\divx{\distr_1}{q} - L\divx{\distr_2}{q}\right| \leq \frac{B}{2}\|\distr_1 - \distr_2\|_1.\]
\end{lemma}

\begin{proof}
    Let $H'(q)$ denote the vector such that $H'(q)_y := \ell(y, q)$ for each $y \in \calY$. We use this vector to equivalently re-write our loss expression as an inner-product over $\calY$:
    \begin{align*}
        |L\divx{\distr_1}{q} - L\divx{\distr_2}{q}| = \left|\langle\distr_1 - \distr_2, H'(q) \rangle\right|.
    \end{align*}
    We note that because $\distr_1$ and $\distr_2$ are both probability distributions over $\calY$, we have that  $\langle \distr_1 - \distr_2, \mathbf{1}\rangle = 0$, and so for any $c \in \R$, we have 

    \begin{align*}
        \left|\langle\distr_1 - \distr_2, H'(q) \rangle\right| &= \left|\langle\distr_1 - \distr_2, H'(q) - c\mathbf{1} \rangle\right|\\
        &\leq \|\distr_1 - \distr_2\|_1 \|H'(q) - c\mathbf{1}\|_{\infty}
    \end{align*}

    We select the $c$ that minimizes $\|H'(q) - c\mathbf{1}\|_\infty$, $c = \frac{1}{2}\left(\max_{y \in \calY}H'(q)_y + \min_{y' \in \calY} H'(q)_{y'}\right)$. This choice of $c$ gives an infinity norm value of 
    \begin{align*}
        \|H'(q) - c\mathbf{1}\|_{\infty} &= \frac{1}{2}\left(\max_{y \in \calY}H'(q)_y - \min_{y' \in \calY} H'(q)_{y'}\right)\\
        &= \frac{1}{2}\left(\max_{y \in \calY}\ell(y, q) - \min_{y' \in \calY} \ell(y', q)\right) \tag{Definition of $H'(q)$}\\
        &\leq \frac{B}{2} \tag{Range of $\ell$ is $[0, B]$.}\\
    \end{align*}

    Substituting this back into our upper bound, we conclude that 
    \[\left|L\divx{\distr_1}{q} - L\divx{\distr_2}{q}\right| \leq \frac{B}{2}\|\distr_1 - \distr_2\|_1.\]
\end{proof}

We use Lemma~\ref{lem:lipschitz-pstar} to prove that two loss-related quantities satisfy $B/2$-Lipschitz continuity. 
\begin{lemma}\label{lem:lip-tl-ir}
    For any fixed $\distr \in \DY$ and bounded proper loss $L:\DY \times \DY \rightarrow [0, B]$, the functions $g_{TL}(\distr^*) = L\divx{\distr^*}{\distr}$ and $g_{IR}(\distr^*) = L\divx{\distr^*}{\distr^*}$ satisfy $B/2$-Lipschitz continuity.
\end{lemma}

\begin{proof}
    The $B/2$-Lipschitz continuity of $g_{TL}$ follows immediately from Lemma~\ref{lem:lipschitz-pstar}. For $g_{IR}$, we note that because $L$ is a proper loss, we have that for any $\distr_1$ and $\distr_2$, 
    \begin{align*}
        g_{IR}(\distr_1) - g_{IR}(\distr_2) &= L\divx{\distr_1}{\distr_1} - L\divx{\distr_2}{\distr_2}\\
        &\leq L\divx{\distr_1}{\distr_2} - L\divx{\distr_2}{\distr_2} \tag{$L$ is proper}\\
        &\leq \frac{B}{2}\|\distr_1 - \distr_2\|_1 \tag{Lemma~\ref{lem:lipschitz-pstar}}
    \end{align*}
    and thus $g_{IR}$ also satisfies $B/2$-Lipschitz continuity. 
\end{proof}

\section{Connecting Optimal Routing Decisions to Reducible Loss}\label{app:optimal-routing-decisions}
In this section, we provide more explanation on how the optimal routing decisions are determined from the amount of reducible and irreducible loss, as illustrated in Figure~\ref{fig:optimal_router_landscape}.

\begin{proposition}
    Given a point $x$ with irreducible loss and reducible loss denoted
    \[\IL := L(\pbayes(x), \pbayes(x)), \quad \RL := L(\pbayes(x), \pred(x)) - L(\pbayes(x), \pbayes(x)),\]
    respectively, the point-wise optimal routing decision for any $\alpha, \beta \geq 0$, 
    \[r^*(x, (L, \alpha, \beta)) = {\arg\min}_{a \in \{P, R, A\}}\mathsf{Cost}(x, a; (L, \alpha, \beta))\]
    is equivalent to the following decision framework in terms of $\IL$ and $\RL$:
    \begin{center}
    \begin{tikzpicture}[
        level 1/.style={sibling distance=7cm, level distance=2cm},
        level 2/.style={sibling distance=3.5cm, level distance=2.5cm},
        decision/.style={rectangle, draw, minimum width=3cm, minimum height=1cm, align=center, thick},
        outcome/.style={rounded corners, draw, fill=gray!10, minimum width=2cm, minimum height=0.8cm, align=center},
        edge label/.style={midway, fill=white, font=\small\itshape}
    ]
    
    \node [decision] {$RL \geq \alpha?$}
        child { node [decision] {$IL \geq \beta - \alpha?$}
            child { node [outcome] {Abstain} 
                edge from parent node[edge label] {Yes} }
            child { node [outcome] {Route} 
                edge from parent node[edge label] {No} }
            edge from parent node[edge label] {Yes}
        }
        child { node [decision] {$IL + RL \geq \beta?$}
            child { node [outcome] {Abstain} 
                edge from parent node[edge label] {Yes} }
            child { node [outcome] {Predict} 
                edge from parent node[edge label] {No} }
            edge from parent node[edge label] {No}
        };
    
    \end{tikzpicture}
    \end{center}
\end{proposition}

\begin{proof}
    Based on our definition of the routing cost, routing has a lower cost than prediction precisely when 
    \begin{align*}
        \mathsf{Cost}(x, R; T) &\leq \mathsf{Cost}(x, P; T)\\
        L(\pbayes(x), \pbayes(x)) + \alpha &\leq L(\pbayes(x), \pred(x))\\
        \alpha &\leq L(\pbayes(x), \pred(x)) - L(\pbayes(x), \pbayes(x)) \\
        \alpha &\leq \RL.
    \end{align*}

    Thus, this reconstructs the first layer of the tree -- routing is preferred over prediction by the optimal router if and only if $\RL \geq \alpha$. 

    Assuming $\RL \geq \alpha$, we conclude that the optimal router will either abstain or route. The property that abstention has a lower cost than routing can be simplified as follows:
    \begin{align*}
        \mathsf{Cost}(x, A; T) &\leq \mathsf{Cost}(x, R; T)\\
        \beta &\leq L(\pbayes(x), \pbayes(x)) + \alpha \\
        \beta - \alpha &\leq \IL
    \end{align*}

    Thus, we have recovered the lower left side of the tree -- assuming $\RL \geq \alpha$, the optimal router abstains whenever $\IL \geq \beta - \alpha$, and routes otherwise. 

    We finally consider the other case: $\RL \leq \alpha$. We know that in this case, the optimal routing policy will either predict or abstain. Abstention has a lower cost compared to prediction precisely when 
    \begin{align*}
        \mathsf{Cost}(x, A; T) &\leq \mathsf{Cost}(x, P; T)\\
        \beta &\leq L(\pbayes(x), \pred(x))\\
        \beta &\leq \RL + \IL.
    \end{align*}

    Thus, we have recovered the lower right side of the tree, and shown that the decisions are exactly the pointwise optimal routing decisions. 
\end{proof}

\section{Proofs and Additional Discussion from Section~\ref{sec:routing_from_hoc}}

\subsection{Proof of Theorem~\ref{thm:hoc_implies_routing}}\label{sec:hoc_implies_routing_pf}

In this section, we prove Theorem~\ref{thm:hoc_implies_routing}, which we restate here for readability:

\hocimpliesrouting*

The proof of the theorem relies on Lemma~\ref{lem:routing_loss_simulation}, which captures why higher-order calibration is so useful when constructing a routing framework. Intuitively, the lemma guarantees that a higher-order-calibrated prediction can be used to obtain good estimates of the routing cost for each potential routing decision. This means that choosing the action that minimizes our estimated cost will also be a good minimizer of the true routing cost.

For ease of exposition, it will be useful to define a generalized routing cost that takes in precise values of $\pbayes(x) = \distr^*$ and $\pred(x) = \distr$, rather than just $x$. We define the generalized cost as follows:
\[\mathsf{Cost}(\distr^*, \distr, a; T) := 
    \begin{cases}
        L\divx{\distr^*}{\distr} & a = P\\
        L\divx{\distr^*}{\distr^*} + \alpha & a = R\\ 
        \beta & a = A
    \end{cases}
\]
Under this new notation, $\mathsf{Cost}(x, a;T)$ can be written equivalently as $\mathsf{Cost}(\pbayes(x), \pred(x), a;T).$

Given a higher-order prediction $\mixpred$, its predictions can be used to estimate the routing cost at each point, where we can interpret the estimates as the routing loss that would be incurred if the true value of $\pbayes(x)$ was actually distributed according to the higher-order prediction:
\begin{definition}[Simulated Routing Cost]\label{def:simulated-routing-cost}
    Given a higher-order predictor $\mixpred$, we define the \emph{simulated routing cost} as 
    \[\hat{\mathsf{Cost}}_{\mixpred}(x, a; T) := \ex_{\distr^* \sim \mixpred(x)}[\mathsf{Cost}(\distr^*, \pred(x), a; T)].\]
\end{definition}

The following lemma establishes that the simulated routing loss is a good estimate for the true routing cost for any action, on average over each partition. 

\begin{lemma}[Routing Loss Prediction from Higher-Order Calibration]\label{lem:routing_loss_simulation}
    Let $\mixpred$ be a $\epsilon$-higher-order-calibrated predictor whose level sets match those of $\pred$ and let $T = (L, \alpha, \beta)$ be a routing configuration where $L$ is a $B$-bounded proper loss $L: \DY \times \DY \rightarrow [0, B]$, and $\alpha, \beta \geq 0$. Then, for any routing decision $a \in \{P, R, A\}$, and higher-order prediction $\mix \in \DDY$ in the range of $\mixpred$, we are guaranteed that 
    \begin{align*}
        \left|\ex_{x \sim \calD_{\calX}}[\mathsf{Cost}(x, a; T)|\mixpred(x) = \mix] - \hat{\mathsf{Cost}}_{\mixpred}(x, a; T)\right| \leq B\epsilon/2.
    \end{align*}
\end{lemma}

We first prove the theorem using the lemma, and then present the proof of the lemma. 

\begin{proof}[Proof of Theorem~\ref{thm:hoc_implies_routing}]
    We define our router $r$ to take the action that minimizes the simulated routing loss computed using $\mixpred$ (Definition~\ref{def:simulated-routing-cost}).

    \[r(x, T) := {\arg\min}_{a \in \{P, R, A\}}\hat{\mathsf{Cost}}_{\mixpred}(x, a; T).\]

    Note that this implies $r$'s routing decisions are constant over each level set of the higher-order predictor (and equivalently $\pred$). 

    Leveraging Lemma~\ref{lem:routing_loss_simulation}, we now show that these actions are also good minimizers for the true routing loss, on average over partitions. Let $\mix \in \DDY$ be some value in the range of $\mixpred$, and fix a configuration $T = (L, \alpha, \beta)$ where $L$ is a $B$-bounded proper loss and $\alpha, \beta \geq 0$. Note that $\pred(x)$ is constant for all points $x$ with $\mixpred(x) = \mix$ because the two predictors have the same level sets. We denote this constant value as $\pred(x) = \distr$. We bound the routing loss over this level set as follows:
    \begin{align*}
        &\ex_{x \sim \calD_{\calX}}[\mathsf{Cost}(x, r(x, T); T)|\mixpred(x) = \mix] \\
        &\leq \hat{\mathsf{Cost}}_{\mixpred}(x, r(x, T); T) + B\epsilon/2 \tag{Lemma~\ref{lem:routing_loss_simulation}}\\
        &\leq \hat{\mathsf{Cost}}_{\mixpred}(x, r_{\pred}(x, T); T)  + B\epsilon/2 \tag{Definition of $r(x, T)$}\\
        &\leq \ex_{x \sim \calD_{\calX}}[\mathsf{Cost}(x, r_{\pred}(x, T); T)|\mixpred(x) = \mix] + B\epsilon \tag{Lemma~\ref{lem:routing_loss_simulation}}\\
    \end{align*}

    Thus, taking the average over all level sets of the higher-order predictor, we conclude that the excess routing cost is bounded:
    \[\ex_{x \sim \calD_{\calX}}[\mathsf{Cost}(x, r(x, T); T)] - \ex_{x \sim \calD_{\calX}}[\mathsf{Cost}(x, r_{\pred}(x, T); T)] \leq B\epsilon.\]

    Because this holds for any configuration $T$ satisfying the conditions of the theorem, the proof is complete. 
\end{proof}

We now finish up by proving our helper lemma. 

\begin{proof}[Proof of Lemma~\ref{lem:routing_loss_simulation}]
    It suffices to prove the statement for an arbitrary $T = (L, \alpha, \beta)$ and predicted mixture $\mix$, where $L$ is a $B$-bounded proper loss, and $\alpha, \beta \geq 0$. Let $\mixbayes$ denote the distribution over $\pbayes(x)$ induced by drawing $x \sim \calD_{\calX}$ conditioned on $\mixpred(x) = \mix$. Note that by our assumption that $\mixpred$ has the same level sets as $\pred$, $\pred(x)$ is constant for all points with $\mixpred(x) = \mix$. We denote this constant value $\pred(x) = \distr$.

    We bound the maximum difference in real vs simulated routing costs by taking a maximum over all possible routing decisions.  We use $\hat \distr \in \DY$ to denote an independent draw from the mixture $\mix$:
    \begin{align*}
        &\max_{a \in \{P, R, A\}}\left|\ex_{x \sim \calD_{\calX}}[\mathsf{Cost}(x, a; T)|\mixpred(x) = \mix] - \hat{\mathsf{Cost}}_{\mixpred}(x, a; T)\right|\\
        &= \max \{\left|\ex_{\distr^* \sim \mixbayes}[L\divx{\distr^*}{\distr}] - \ex_{\hat\distr \sim \mix}[L\divx{\hat\distr}{\distr}]\right|, \left|\ex_{\distr^* \sim \mixbayes}[L\divx{\distr^*}{\distr^*}] - \ex_{\hat\distr \sim \mix}[L\divx{\hat\distr}{\hat\distr}]\right|, 0\}
    \end{align*}

    We use the following dual definition of the 1-Wasserstein distance. Let $\mathcal{G}_{Lip}$ be the class of all $1$-Lipschitz functions $g: \DY \rightarrow \R$. We can equivalently write the Wasserstein distance between $\mix_1 \in \DDY$ and $\mix_2 \in \DDY$ as 
    \[W_1(\mix_1, \mix_2) = \max_{g \in \mathcal{G}_{Lip}}\ex_{\distr_1 \sim \mix_1}[g(\distr_1)] - \ex_{\distr_2 \sim \mix_2}[g(\distr_2)].\]

    We note that by Lemma~\ref{lem:lip-tl-ir}, the functions $g_{TL}(\distr^*) = L\divx{\distr^*}{\distr}$ and $g_{IR}(\distr^*) = L\divx{\distr^*}{\distr^*}$ both satisfy $B/2$-Lipschitzness, and thus we can further bound our maximum by a multiplicative factor of the maximum over all 1-Lipschitz functions:
    \begin{align*}
        \frac{B}{2}\max_{g \in \mathcal{G}_{Lip}}\ex_{\distr^* \sim \mixbayes}[g(\distr^*)] - \ex_{\hat\distr \sim \mix}[g(\hat\distr)]&= \frac{B}{2}W_1(\mixbayes, \mix) \tag{Wasserstein duality}\\
        &\leq B\epsilon/2 \tag{Higher-order calibration guarantee}
    \end{align*}

    Thus, we conclude that the difference between the true and estimated routing loss values is at most $B\epsilon/2$.
\end{proof}

\subsection{Generalizing to Additional Oracles}\label{app:bayes-dep-oracles}

While we state our main result for a particular routing setting where the oracle makes Bayes-optimal predictions, our proof techniques can naturally be extended to settings where we have access to multiple oracles, each with varying cost. The key requirement we make is that each oracle's loss at a point $x$ can be expressed as a known function of $\pbayes(x)$. We call such an oracle a \emph{Bayes-dependent oracle}. 

\begin{definition}[Bayes-Dependent Oracle]
    An oracle $g_{O}(x): \calX \rightarrow \DY$ is \emph{Bayes-Dependent} if there exists a function $c: \mathcal{L} \times \DY \rightarrow \R$ such that for any point $x \in \calX$, $L(\pbayes(x), g_{O}(x)) = c(L, \pbayes(x))$, i.e. the expected loss of the oracle at $x$ can be expressed purely as a function of $\pbayes(x)$ and the loss of interest. 
\end{definition}

We note that this definition captures a wide range of oracles that we might have access to in practice. As one example, a common alternative option to prediction might be to ask one or more human experts for their opinion on how to classify a particular point. Each human expert's opinion can be viewed as a draw $y \sim \pbayes(x)$ from the true conditional distribution. 

Thus, the expected loss of querying $k$ human experts and combining their opinions using some aggregation function $a: \calY^k \rightarrow \DY$ can be expressed as a function of $\pbayes(x)$:
\[c(L, \pbayes(x)) = \ex_{y_1, ..., y_k \sim \pbayes(x)}[L(\pbayes(x), a(y_1, ..., y_k))].\]

We can thus expand our routing setting to assume that we have access to $k$ oracles whose losses are computed by (known) functions $c_1, ..., c_k$. Our action set also expands to $\{P, A, R_1, ..., R_k\}$, where $R_i$ denotes routing to the $i$th oracle, and we have $k$ routing penalties $\alpha_1, ..., \alpha_k \geq 0$ in addition to the abstention penalty $\beta \geq 0$. 

In such a setting, we can naturally derive a generalized version of Lemma~\ref{lem:routing_loss_simulation}, where the simulated costs for action $R_i$ are now computed as 

\[\hat{\mathsf{Cost}}_F(x, R_i; T) := \ex_{\distr^* \sim F(x)}[c_i(L, \distr^*)] + \alpha_i.\]

\begin{lemma}\label{lem:generalized-simulated-cost}
    Let $\mixpred$ be a $\epsilon$-higher-order-calibrated predictor whose level sets match those of $\pred$ and let $T = (L, \alpha_{1:k}, \beta)$ be a routing configuration where $L$ is a $B$-bounded proper loss $L: \DY \times \DY \rightarrow [0, B]$, and $\alpha_{1:k}, \beta \geq 0$. Let each $c_i(L, \cdot)$ satisfy $\gamma$-Lipschitz continuity. 
    
    Then, for any routing decision $a \in \{P, R_1, ..., R_k, A\}$, and higher-order prediction $\mix \in \DDY$ in the range of $\mixpred$, we are guaranteed that 
    \begin{align*}
        \left|\ex_{x \sim \calD_{\calX}}[\mathsf{Cost}(x, a; T)|\mixpred(x) = \mix] - \hat{\mathsf{Cost}}_{\mixpred}(x, a; T)\right| \leq \max\{B/2, \gamma\}\epsilon.
    \end{align*}
\end{lemma}

The proof follows an almost identical argument to Lemma~\ref{lem:routing_loss_simulation}, which we present below. With this lemma in hand, we can apply it to get a generalized version of our main result, Theorem~\ref{thm:hoc_implies_routing}. We present only the statement without proof, as the proof is identical to the proof of Theorem~\ref{thm:hoc_implies_routing}, plugging in the updated bound  of $\max\{B/2, \gamma\}\epsilon$ on the gap between true and simulated routing costs. 

\begin{theorem}[Prediction-Optimal Routing with Bayes-Dependent Oracles (Generalization of Theorem~\ref{thm:hoc_implies_routing})]
    Suppose we have an $\epsilon$-higher-order calibrated predictor $\mixpred$ whose level sets match those of $\pred$ (i.e., $\mixpred$ can be viewed as a function of $\pred$). Then we can construct a flexible router $r$ based on $\mixpred$ with the following near-optimality property. For any routing configuration $T = (L, \alpha_{1:k}, \beta)$ where $L$ is a $B$-bounded proper loss, $\alpha_{1:k}, \beta \geq 0$, and each routing decision $R_i$ routes to a Bayes-dependent oracle whose loss can be computed using the $\gamma$-Lipschitz function $c_i(L, \cdot)$, then
    we have \[\ex_{x \sim \calD_{\calX}}[\mathsf{Cost}(x, r(x, T); T) - \mathsf{Cost}(x, r_{\pred}(x, T); T)] \leq \max\{B\epsilon, 2\gamma\epsilon\}.\]
\end{theorem}

We now provide the proof of Lemma~\ref{lem:generalized-simulated-cost}.
\begin{proof}[Proof of Lemma~\ref{lem:generalized-simulated-cost}]
    It suffices to prove the statement for an arbitrary $T = (L, \alpha_{1:k}, \beta)$ and predicted mixture $\mix$, where $L$ is a $B$-bounded proper loss, and $\alpha_{1:k}, \beta \geq 0$. Let $\mixbayes$ denote the distribution over $\pbayes(x)$ induced by drawing $x \sim \calD_{\calX}$ conditioned on $\mixpred(x) = \mix$. Note that by our assumption that $\mixpred$ has the same level sets as $\pred$, $\pred(x)$ is constant for all points with $\mixpred(x) = \mix$. We denote this constant value $\pred(x) = \distr$.

    We bound the maximum difference in real vs simulated routing costs by taking a maximum over all possible routing decisions:
    \begin{align*}
        &\max_{a \in \{P, R_{1:k}, A\}}\left|\ex_{x \sim \calD_{\calX}}[\mathsf{Cost}(x, a; T)|\mixpred(x) = \mix] - \hat{\mathsf{Cost}}_{\mixpred}(x, a; T)\right|\\
        &= \max \{\left|\ex_{\distr^* \sim \mixbayes}[L\divx{\distr^*}{\distr}] - \ex_{\hat\distr \sim \mix}[L\divx{\hat\distr}{\distr}]\right|, 0 , \max_{i \in [k]}\{\left|\ex_{\distr^* \sim \mixbayes}[c_i(L, \distr^*)] - \ex_{\hat{\distr} \sim \mix}[c_i(L, \hat{\distr})]\right|\}\}
    \end{align*}

    We use the following dual definition of the 1-Wasserstein distance. Let $\mathcal{G}_{Lip}$ be the class of all $1$-Lipschitz functions $g: \DY \rightarrow \R$. We can equivalently write the Wasserstein distance between $\mix_1 \in \DDY$ and $\mix_2 \in \DDY$ as 
    \[W_1(\mix_1, \mix_2) = \max_{g \in \mathcal{G}_{Lip}}\ex_{\distr_1 \sim \mix_1}[g(\distr_1)] - \ex_{\distr_2 \sim \mix_2}[g(\distr_2)].\]

    We note that by Lemma~\ref{lem:lip-tl-ir}, the function $g_{TL}(\distr^*) = L\divx{\distr^*}{\distr}$ satisfies $B/2$-Lipschitzness, and by assumption each $c_i(L, \cdot)$ satisfies $\gamma$-Lipschitz-ness, and thus we can further bound our maximum by a multiplicative factor of the maximum over all 1-Lipschitz functions:
    \begin{align*}
        \max\{\gamma, \frac{B}{2}\}\max_{g \in \mathcal{G}_{Lip}}\ex_{\distr^* \sim \mixbayes}[g(\distr^*)] - \ex_{\hat\distr \sim \mix}[g(\hat\distr)]&= \max\{\frac{B}{2}, \gamma\}W_1(\mixbayes, \mix) \tag{Wasserstein duality}\\
        &\leq \max\{B/2, \gamma\}\epsilon \tag{Higher-order calibration guarantee}
    \end{align*}

    Thus, we conclude that the difference between the true and estimated routing loss values is at most $\max\{B/2, \gamma\}\epsilon$.
\end{proof}

\section{Characterizing Useful Partitions}\label{app:choosing-good-partitions}

As discussed in Section~\ref{sec:practical-discussion}, the choice of partition $\calP$ is critical to the success of our method. Even among all partitions that allow for sample-efficient estimation of the loss decomposition, some are better than others. The following lemma gives a property of a partition that characterizes how the best partition-dependent decisions compare to the pointwise-optimal decisions: partitions that group points with similar amounts of reducible loss have routing performance very close to the pointwise-optimal cost. 

For simplicity, we focus on the setting where $\beta = \infty$, i.e. the routing policy need only choose between prediction and routing to the oracle. However, using similar techniques an expression can be derived for general $\beta$, and is a more complicated expression capturing the variance in both the reducible and irreducible loss within each partition cell. 

\begin{lemma}
    Given a partition $\calP$ of the space $\calX$ and a routing configuration $T = (L, \alpha, \beta)$ for a proper loss $L$ and $\alpha \geq 0$, $\beta = \infty$ (no abstention), consider a bin within the partition, $P \in \calP$. Denote $\RL$ as the random variable corresponding to the reducible loss of points $x \sim \calD_{\calX}$, conditioned on falling within $P$, with mean $\overline{\RL}$. 
    
    Then, the excess routing cost incurred by picking the best single routing decision for all of $P$ compared to the pointwise-optimal router has the following upper bound:
    \[\min_{a \in \{P, R, A\}}\ex_{x \sim \calD_{\calX}}[\mathsf{Cost}(x, a; T) | x \in P] - \ex_{x \sim \calD_{\calX}}[\mathsf{Cost}(x, r^*(x, T); T)| x \in P] \leq \frac{1}{2}\ex[\left|\RL - \overline{\RL}\right|].\]
\end{lemma}

We note that this bound is tight in that there exists an $\alpha$ that results in excess cost matching the upper bound.
\begin{proof}
    Similar to $\RL$, let $\IL$ denote the random variable corresponding to the irreducible loss for a point $x \sim \calD_{\calX}$ conditioned on $x \in P$, with mean $\overline{\IL}$. 

    Substituting these values into our definition of $\mathsf{Cost}$, and noting that because $\beta = \infty$, it is never optimal to abstain, the excess cost expression simplifies to 
    \begin{align*}
        \min \{\overline{\IL} + \overline{\RL}, \overline{\IL} + \alpha\} - \ex[\min\{\IL + \RL, \IL + \alpha\}] &= \min \{\overline{\IL} + \overline{\RL}, \overline{\IL} + \alpha\} - \ex[\IL + \min\{\RL,\alpha\}] \\
        &= \min\{\overline{\RL}, \alpha\} - \ex[\min\{\RL, \alpha\}]].
    \end{align*}

    When $\alpha \leq \overline{\RL}$, this expression is increasing with $\alpha$, and when $\alpha \geq \overline{\RL}$, the expression is non-increasing as $\alpha$ increases. Thus, we conclude that $\alpha^* = \overline{\RL}$ maximizes the expression. We substitute in $\alpha^*$ to get an upper bound on the excess cost, which simplifies the excess cost to 
    \begin{align*}
        \overline{\RL} - \ex[\min\{\RL, \overline{\RL}\}]&= \ex[(\overline{\RL} - \RL)^+]\\
        &= \frac{1}{2}\ex[\left|\overline{\RL} - \RL\right|].
    \end{align*}
\end{proof}

\subsection{Bucketing Strategies}
With this intuition in mind, we recommend the following strategies for partitioning:

\subsubsection{Top-Class Binning} Calibrating on the full vector of probability outputs is intractable for multiclass problems due to the exponential growth of the output space. To maintain sample-efficiency, we partition the space based on the \emph{top-class prediction} and its associated confidence (e.g., ``Class 5, Confidence 0.9-1.0''). This ensures the partition size scales linearly with the number of classes, making calibration more sample-efficient.

\subsubsection{Centroid Calibration} In the ``Optional Re-Calibration Step,'' we post-process the weak predictor to output the centroid of the empirical distribution of true labels in its assigned bin. This alignment ensures the predictor is constant over each level set of the higher-order predictor, satisfying the assumptions of Theorem~\ref{thm:hoc_implies_routing}. In practice, we found this not only stabilizes routing but often reduces the weak model’s original prediction loss, as it is essentially a form of (first-order) post-hoc calibration.

\begin{remark}
    We note that if the ``weak'' model is already highly performant, collapsing its output to the centroids of a coarse partition may discard useful signal. In such cases, this step should be skipped.
\end{remark}

\subsubsection{Feature-based Binning} Finally, relying solely on the weak model's confidence may limit a partition's usefulness. In practice, it may be useful to additionally define bins using auxiliary features that can be expected to correlate with the amount of irreducible uncertainty, such as a heuristic for the blurriness of an image in an image classification task. We provide a synthetic example of how feature-based binning can improve over binning by prediction value in \cref{fig:piecewise_recal_curve}.

\section{Additional Experiment Discussion}

\subsection{Dataset Details}\label{sec:dataset-details}
We detail each dataset and associated weak predictor architecture below.
We do not perform any hparam optimization since we are not trying to obtain the best performing model in any setting, we mostly use defaults from other implementations.

\begin{itemize}
\item \textbf{Synthetic Data}: A set of synthetic data functions with baseline aleatoric uncertainty inspired by \citet{Johnson24}. We give further explanation of how we generate this data in \cref{sec:synthetic_data_gen}. For the weak predictor, we train a simple ResNet with three residual blocks and two MLP layers in each block, following the architecture used by \citet{Johnson24}. The model achieves nearly perfect fit to the ground-truth function, since we can generate infinite synthetic examples for training.
    
\item \textbf{CIFAR-10H}: A re-annotation of the CIFAR-10 test set where each image has at least 47 human labels \citep{peterson2019human}. As the weak predictor, we train a standard ResNet34 \citep{he2016deep} on the original CIFAR-10 training split.
Our model achieves 87\% accuracy on the original CIFAR-10 test split, consistent with being a weak, imperfect predictor.
We split the 10K example CIFAR-10H test-set into 5K for post-hoc higher order calibration, and 5K hold-out examples for testing the router. 
Dataset License: CC BY-NC-SA 4.0
Further hparams:
\begin{itemize}
    \item \textbf{Training epochs:} $50$.
    \item \textbf{Batch size:} $512$.
    \item \textbf{Optimizer:} AdamW.
    \item \textbf{Maximum learning rate:} $3.799 \times 10^{-3}$.
    \item \textbf{Weight decay:} $3.656 \times 10^{-1}$.
    \item \textbf{AdamW defaults:} $\beta_1=0.9$, $\beta_2=0.999$, $\epsilon=10^{-8}$.
    \item \textbf{Learning-rate schedule:} OneCycleLR stepped once per minibatch, with $50$ epochs and $\texttt{steps\_per\_epoch}=|\mathcal{D}_{\mathrm{train}}|/512$.
    \item \textbf{Loss:} cross-entropy.
    \item \textbf{Gradient clipping:} value clipping at $0.1$.
\end{itemize}

\item \textbf{SNLI}: Natural language inference task, classify pairs of sentences into ``contradiction'', ``neutral'', or ``entailment''. Five human labels per example \citep{bowman2015large}. We finetune a 100M parameter Bert model \citep{bert} (License: apache-2.0) on the classification task given pairs of sentences, using the majority label. Training is done on the original SNLI train set. 
The resulting test accuracy of the weak predictor is 85\%.
We run routing experiments by splitting the SNLI 10K test set, with 5K examples used for post-hoc higher order calibration, and 5K hold-out examples used for testing the router.
Dataset License: CC BY-SA 4.0. Further hparams:
\begin{itemize}
    \item \textbf{SNLI model:} \texttt{bert-base-uncased} sequence classifier with $3$ output labels.
    \item \textbf{Training epochs:} $3$.
    \item \textbf{Batch size:} $32$.
    \item \textbf{Optimizer:} AdamW.
    \item \textbf{Learning rate:} $2 \times 10^{-5}$.
    \item \textbf{AdamW defaults:} $\beta_1=0.9$, $\beta_2=0.999$, $\epsilon=10^{-8}$, weight decay $0.01$.
    \item \textbf{Learning-rate schedule:} linear decay with warmup.
    \item \textbf{Warmup:} $10\%$ of total training steps.
    \item \textbf{Loss:} cross-entropy via the HuggingFace sequence-classification loss.
    \item \textbf{Gradient clipping:} norm clipping at $1.0$.
    \item \textbf{Maximum sequence length:} $128$ tokens.
\end{itemize}

\item \textbf{ChaosNLI}: Re-labeling of SNLI and MNLI \citep{N18-1101} with 100 human labels per sentence pair \citep{ynie2020chaosnli}, increasing the amount of aleatoric uncertainty. 
As the weak predictor, we utilize the original SNLI/MNLI predictions given by the authors \citep{ynie2020chaosnli}. These were generated by finetuning a RoBERTa-large model on a variety of NLI tasks \citep{nie-etal-2020-adversarial}.
We use a 3113 subset of the ChaosNLI dataset belonging to SNLI and MNLI, and further use half the data for post-hoc model higher-order calibration, and the other half for routing experiments.
Dataset License: CC BY-NC 4.0

\end{itemize}

All predictions for routing were generated in <1 hour on a single A100 GPU.
Running routing algorithms themselves is cheap and CPU-bound.

\subsection{Synthetic Data Generation Details}\label{sec:synthetic_data_gen}

We employ a synthetic binary classification setting to precisely control the learning environment, particularly the irreducible uncertainty. The inputs are real-valued scalars drawn from a standard normal distribution, $x \sim \mathcal{N}(0, 1)$. We define the target distribution via an explicit Bayes-optimal predictor $\pbayes: \mathbb{R} \rightarrow [0, 1]$, which maps each input $x$ to the true conditional probability $P(y = 1 \mid x)$. By explicitly designing $\pbayes$, we can directly manipulate both the magnitude and the variance of the irreducible uncertainty across the input space.

For a given $\pbayes$, we train the weak predictor on a dataset of 10,000 samples $(x, y)$, where $y \sim \mathsf{Ber}(\pbayes(x))$. To perform higher-order calibration, we generate a $k$-snapshot calibration set containing 5,000 instances. For each input $x$, we sample $k=100$ independent and identically distributed labels $y_i \sim \mathsf{Ber}(\pbayes(x))$. Finally, we evaluate all methods on a held-out test set of 500,000 samples.

\cref{fig:synthetic_data_fns} visualizes the three target distributions ($\pbayes$) we consider, the predictions of the trained weak models, and the post-hoc calibrated predictions used in our experiments. The precise functions used to construct each of these examples is detailed below:

\begin{itemize}
    \item \cref{fig:johnson_fn}: 
    $$\pbayes(x) = \frac{0.98 u(x) + 1}{2}$$where the intermediate variables $u(x)$, $v(x)$, and $w(x)$ are defined as:$$u(x) = 0.6 \cos(v(x)) + 0.4 \cos(4.2 x)$$$$v(x) = \text{sgn}(x) \left(120|x| - 112w(x) - 0.0635\right)$$$$w(x) = 0.2 \ln\left(1 + \exp\left(\frac{|x| - 1}{0.2}\right)\right)$$
    
    \item \cref{fig:three_steps_fn}: 
    $$ \pbayes(x) = \begin{cases} 
    0 & \text{if } x \le -1 \\ 
    \frac{\sin(100x)}{2} + 0.5 & \text{if } -1 < x < 1 \\ 
    1 & \text{if } x \ge 1 
    \end{cases} $$
    \item \cref{fig:piecewise_fn}: 
    $$ \pbayes(x) = \begin{cases}
    0.5 & \text{if } x \le -1 \\
    \frac{\sin(100x)}{4} + 0.5 & \text{if } -1 < x \le -0.5 \\
    0.25 & \text{if } -0.5 < x \le 0 \\
    \frac{\sin(100x)}{4} + 0.5 & \text{if } 0 < x \le 0.5 \\
    \frac{\sin(100x)}{4} + 0.25 & \text{if } x > 0.5
    \end{cases} $$
\end{itemize}

\begin{figure}[htbp]
     \centering
     \begin{subfigure}[b]{0.3\textwidth}
         \centering
         \includegraphics[width=\textwidth]{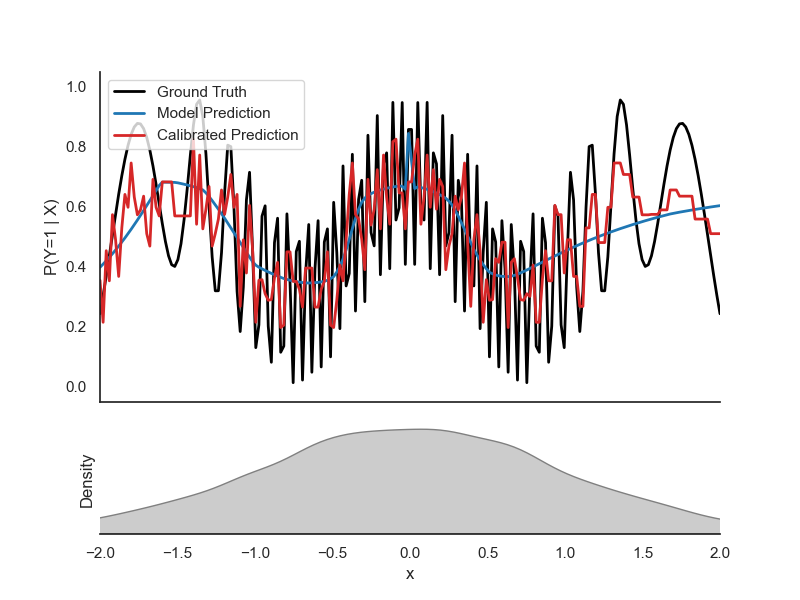}
         \caption{Sinusoidal ground truth with variable irreducible uncertainty, matching the test function used in the synthetic experiments of \citeauthor{Johnson24}.}
         \label{fig:johnson_fn}
     \end{subfigure}
     \hfill
     \begin{subfigure}[b]{0.3\textwidth}
         \centering
         \includegraphics[width=\textwidth]{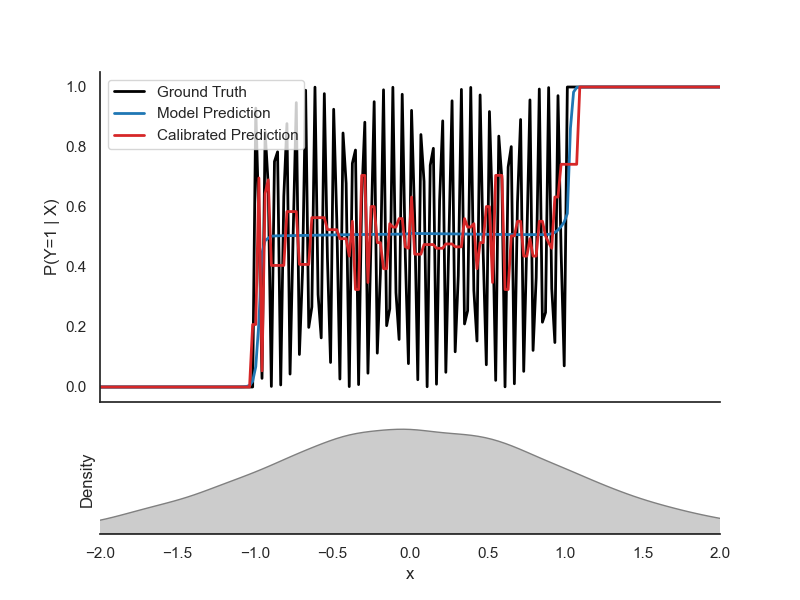}
         \caption{Step function distribution with varying amounts of irreducible uncertainty. \vspace{24pt}}
         \label{fig:three_steps_fn}
     \end{subfigure}
     \hfill
     \begin{subfigure}[b]{0.3\textwidth}
         \centering
         \includegraphics[width=\textwidth]{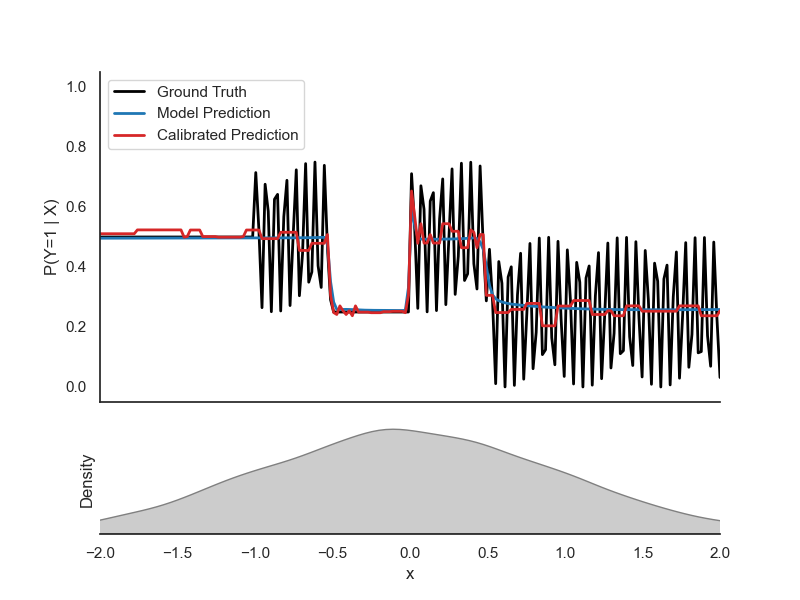}
         \caption{Mixed target distribution where regions of low and high variance in the amount of irreducible uncertainty share the same mean.}
         \label{fig:piecewise_fn}
     \end{subfigure}
     
     \caption{Examples of synthetic binary data settings used in experiments, as described in \cref{sec:synthetic_data_gen}. The marginal distribution over $x$ is the standard normal distribution, and the black curve represents the true conditional expectation of $y$ for each $x$, i.e. $\pbayes(x)$. The blue curve captures the predictions of a model trained on the data, and the red curve captures the predictions of the model after post-hoc calibration on the same buckets used by our routing framework.}
     \label{fig:synthetic_data_fns}
\end{figure}

\subsection{Benefits of Decomposition}

\subsubsection{Further Discussion of Decomposition}\label{sec:decomp-discussion}

We highlight two key factors that dictate the benefits of our method, and more broadly any method that can separate irreducible and reducible loss:

\paragraph{Reducible vs. Total Loss Correlation} In environments where reducible and total loss are highly correlated, a ranking based on total uncertainty serves as a near-perfect proxy for a ranking by reducible loss. In such cases, simpler confidence-based metrics are sufficient for near-optimal routing. We note that this correlation may be an inherent property of the dataset, such as if all points have a similar amount of irreducible loss, or can be a byproduct of the training dynamics used to obtain the weak predictor. 

We provide an example of the latter case in \cref{fig:three_steps_recal_curve}. While the target function (\cref{fig:three_steps_fn}) has significant variance in the amount of irreducible uncertainty across inputs, the highly varied portion of the data (where the model experiences the most irreducible loss) corresponds with the places where a weak model is likely to predict a value close to 0.5, the maximal amount of total uncertainty, causing routing based on total uncertainty to perform almost as well as routing based on the reducible loss. 

\begin{figure*}[ht]
    \centering   
    \includegraphics[width=0.75\linewidth]{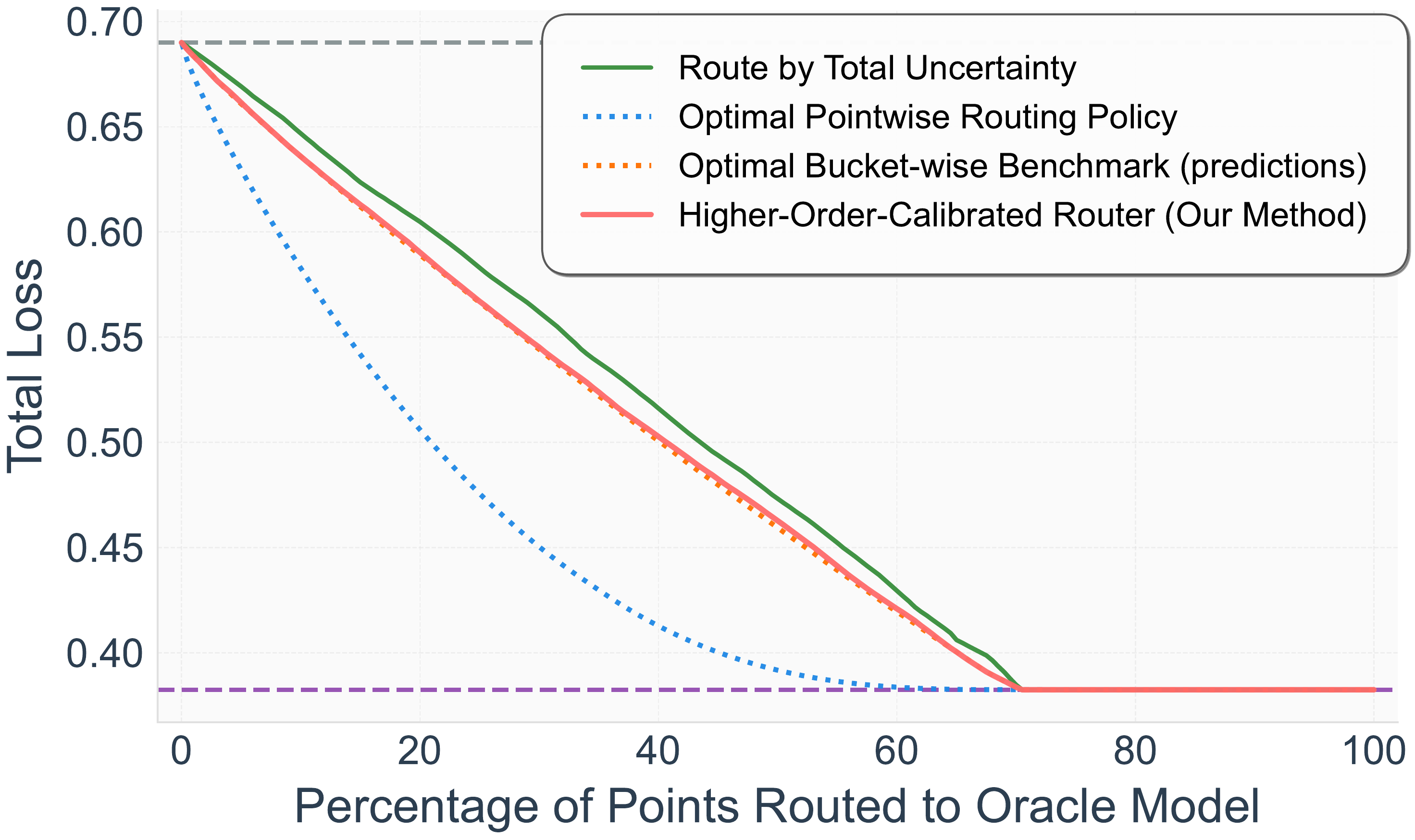}
    \caption{The above routing curve depicts the performance of our routing framework on the synthetic target function visualized in \cref{fig:three_steps_fn}. We highlight this as an example of a setting in which there is a large amount of variation in the irreducible uncertainty of the data, but training dynamics result in a model that doesn't significantly benefit from decomposition, as evidenced by the fact that the performance achieved by a total uncertainty-based approach is closely aligned with the routing curve of our higher-order method.}
    \label{fig:three_steps_recal_curve}
\end{figure*}

\paragraph{The Impact of Bucket Structure} As discussed in \cref{sec:practical-discussion}, for bucket-based approaches like ours, the choice of partition is critical. Even if reducible and total loss are uncorrelated in a pointwise sense, a poorly chosen bucketing scheme may average these values in a way that induces a high correlation per-bucket. To illustrate this, \cref{fig:piecewise_recal_curve} provides a synthetic example where introducing even a small number of feature-based buckets results in significant performance gains from decomposition compared to bucketing only by level sets.

\begin{figure*}[ht]
    \centering   
    \includegraphics[width=0.75\linewidth]{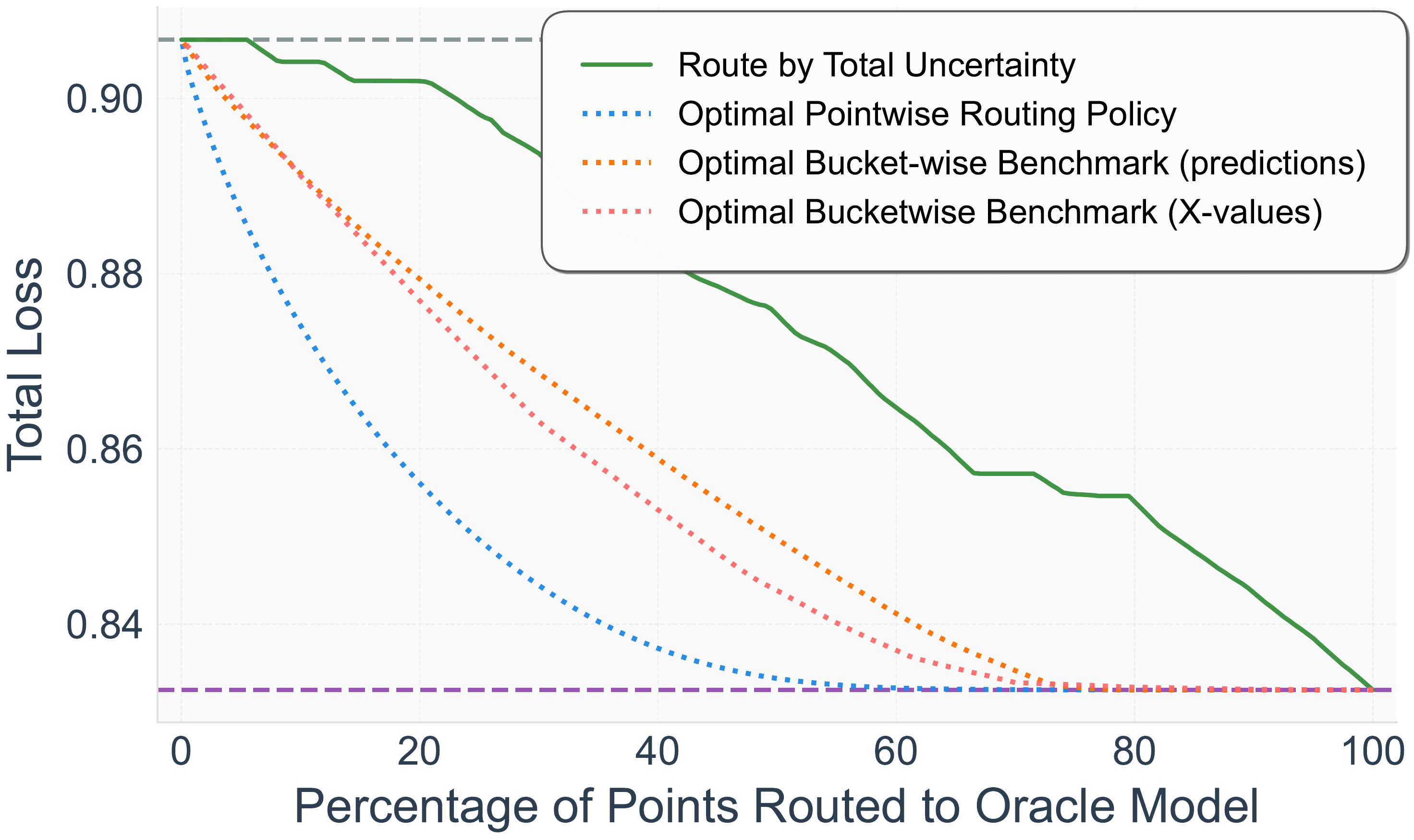}
    \caption{The above routing curve compares the performance of two bucketing approaches on synthetic data with target function depicted in \cref{fig:piecewise_fn}. The orange and pink lines both depict the optimal performance of a router that must be constant over a particular partition of the input space of buckets. The orange line is the performance when 30 buckets of equal size are created based on the level sets (predictions) of the weak model, while the pink line results from using just 10 equal-sized buckets that split on the input value $x$. The improved performance of the pink line illustrates how the particular choice of bucketing can be critical the performance of our routing framework.}
    \label{fig:piecewise_recal_curve}
\end{figure*}

\section{Routing Curve Plots}\label{sec:routing-curve-plots}
This section provides the full routing curve results for all evaluated datasets, comparing the calibrated and uncalibrated models.
\subsection{Synthetic Data}\label{sec:routing-plots-synthetic}
\compareplots
    {johnson_recal.pdf} %
    {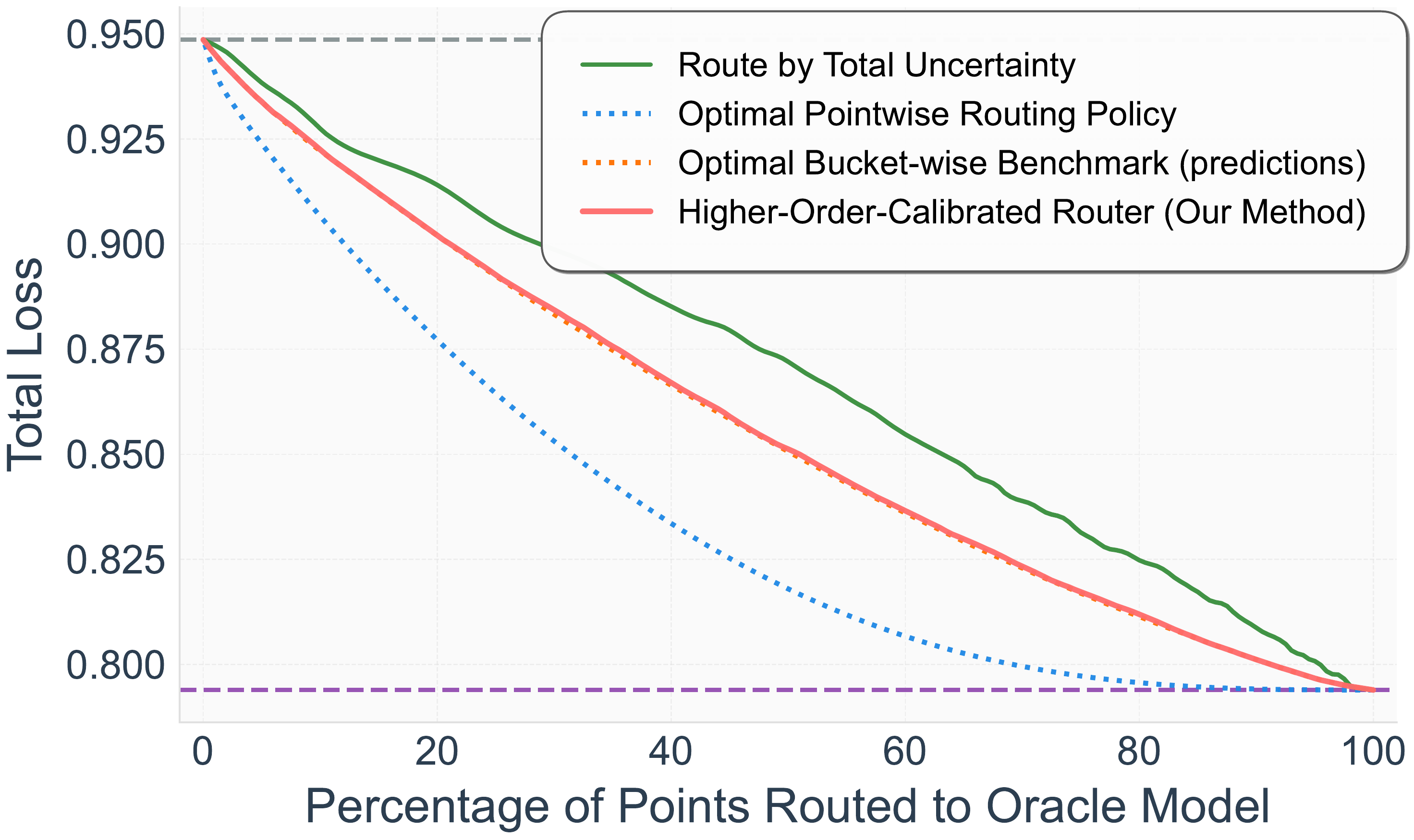} %
    {fig:routing_johnson_recal} %
    {fig:routing_johnson_no_pp} %
    {Routing curves for synthetic data as described in \cref{sec:synthetic_data_gen} and target function as depicted in \cref{fig:johnson_fn}. The right (b) employs an out-of-the-box weak predictor for the dataset, whereas the left (a) uses the same predictor but with an additional post-hoc calibration step on the same buckets used by the higher-order-calibrated router.} %
    {fig:routing_johnson_full}

\subsection{CIFAR-10H}

\compareplots
    {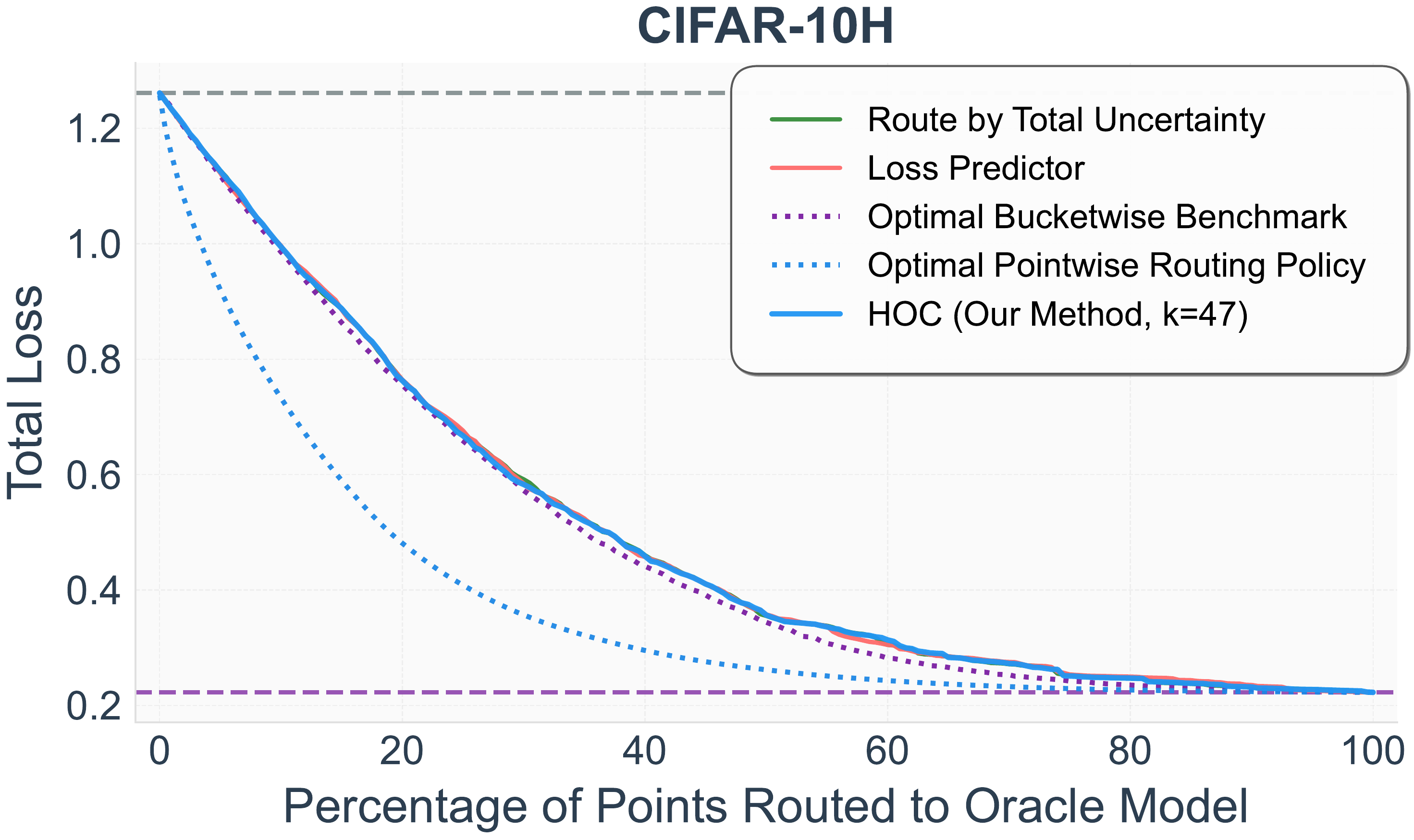} %
    {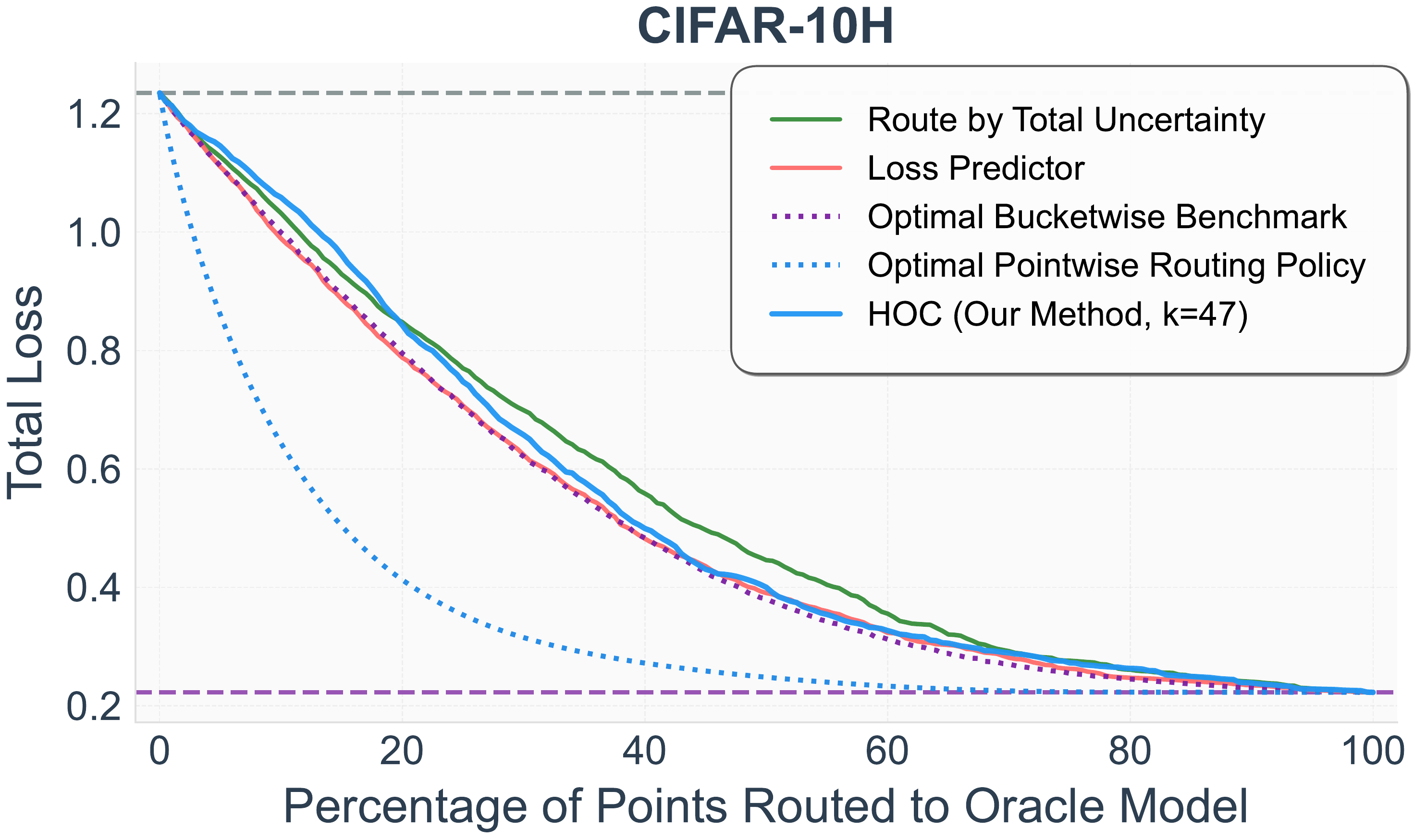} %
    {fig:routing_cifar10_recal} %
    {fig:routing_cifar10_no_pp} %
    {Routing curves for the CIFAR-10H dataset. The right (b) employs an out-of-the-box weak predictor for the dataset, whereas the left (a) uses the same predictor but with an additional post-hoc calibration step on the same buckets used by the higher-order-calibrated router.} %
    {fig:routing_cifar10_full}
    
\subsection{SNLI}

\compareplots
    {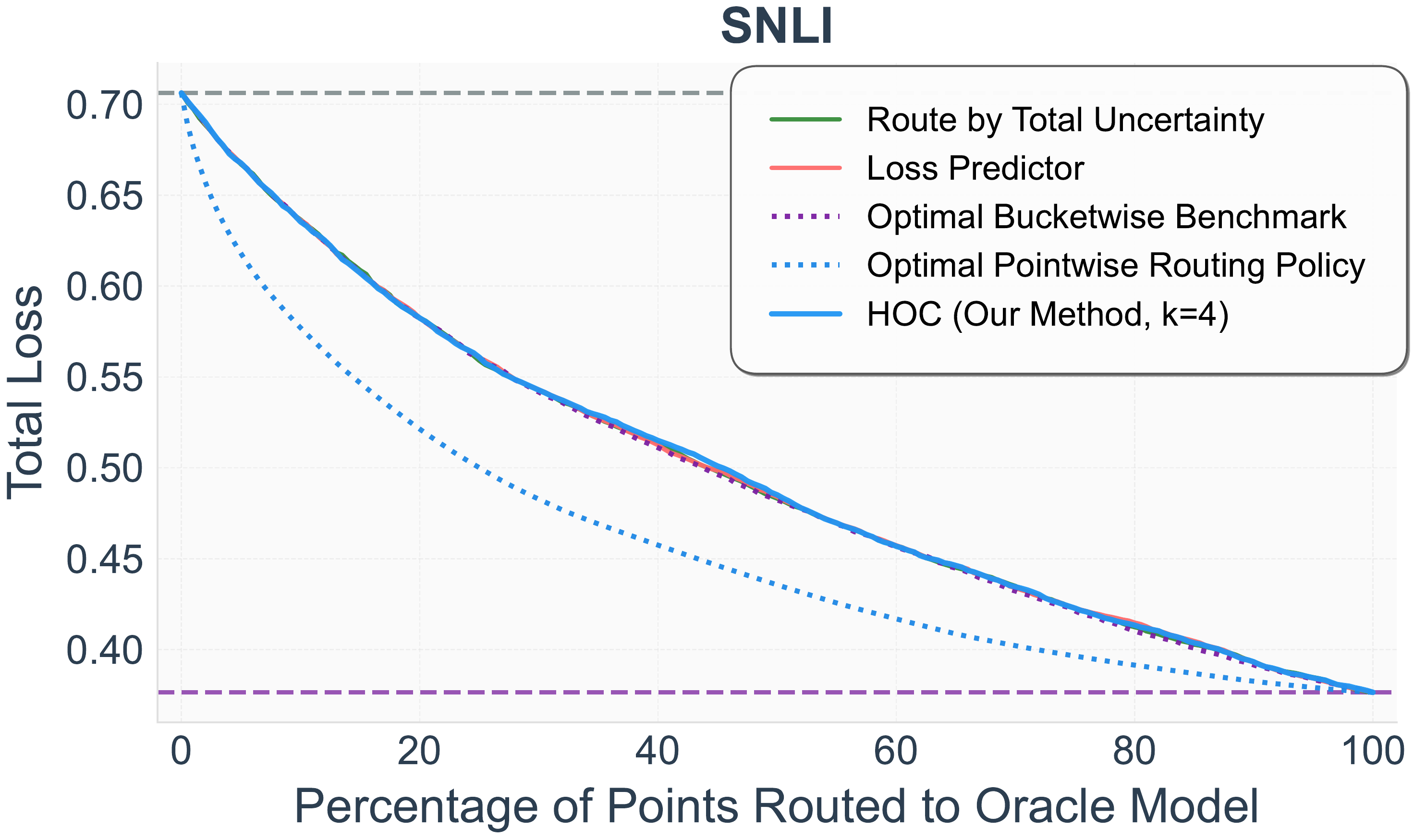} %
    {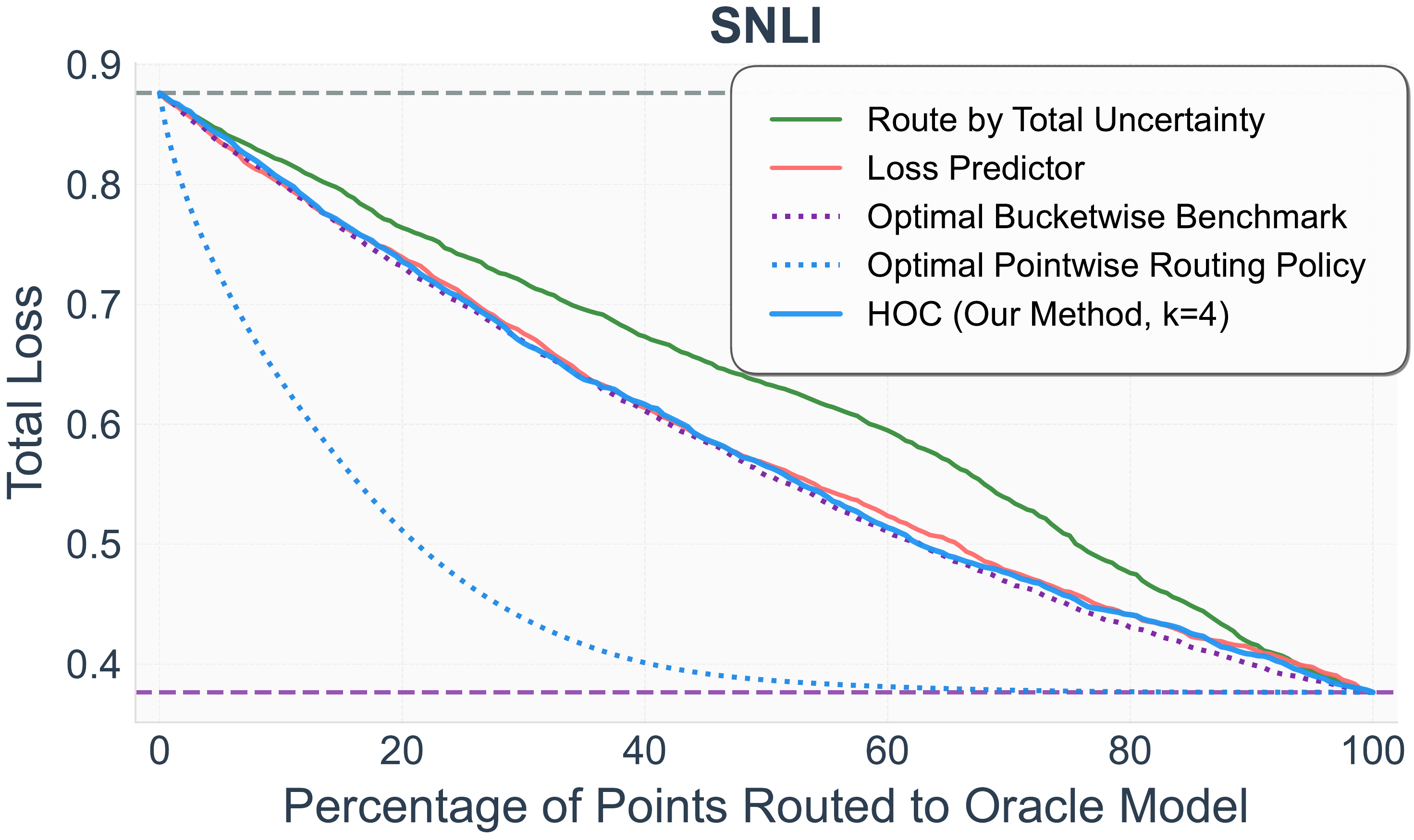} %
    {fig:routing_snli_recal} %
    {fig:routing_snli_no_pp} %
    {Routing curves for the SNLI dataset. The right (b) employs an out-of-the-box weak predictor for the dataset, whereas the left (a) uses the same predictor but with an additional post-hoc calibration step on the same buckets used by the higher-order-calibrated router.} %
    {fig:routing_snli_full}
    
\subsection{ChaosNLI}

\compareplots
    {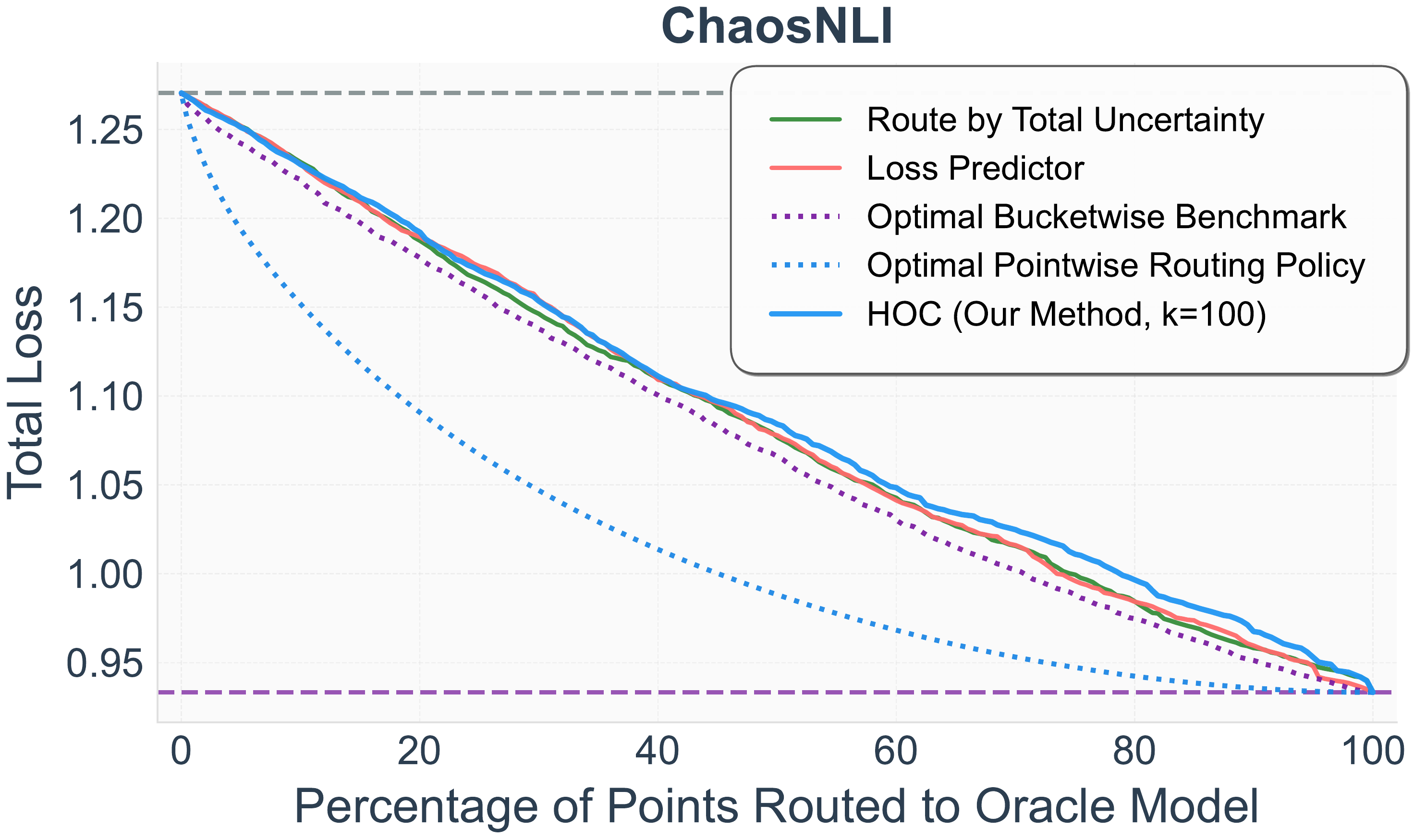} %
    {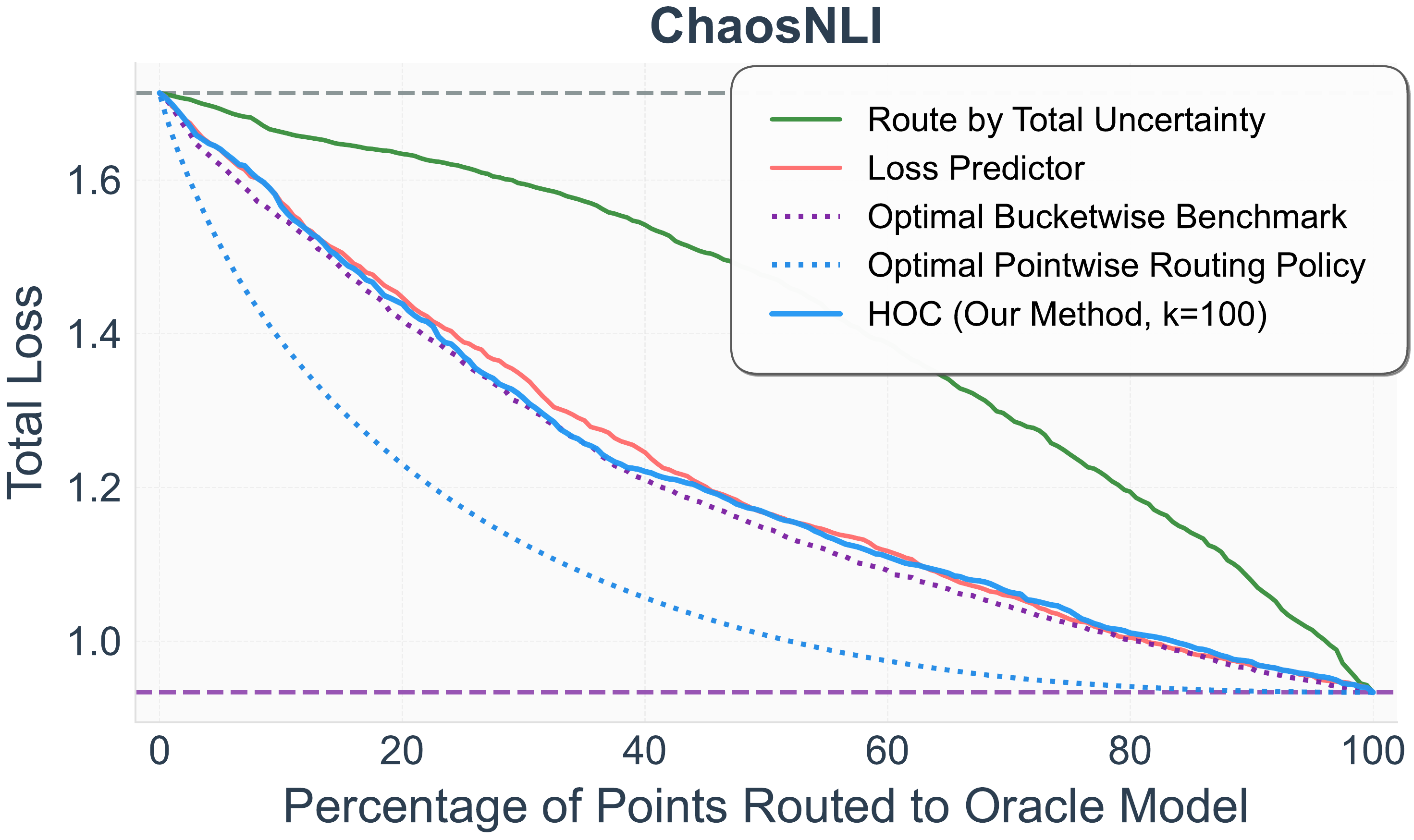} %
    {fig:routing_chaosnli_recal} %
    {fig:routing_chaosnli_no_pp} %
    {Routing curves for the ChaosNLI dataset. The right (b) employs an out-of-the-box weak predictor for the dataset, whereas the left (a) uses the same predictor but with an additional post-hoc calibration step on the same buckets used by the higher-order-calibrated router.} %
    {fig:routing_chaosnli_full}

\clearpage

\section{Loss Flexibility Plots}\label{sec:loss-comp-plots}
\subsection{Explanation of Loss Functions}\label{sec:loss-explanation}
We overview the different losses appearing in the comparison plots in this section.

\textbf{Remark on Proper Losses:} Our routing framework can flexibly adapt to various new losses, with the restriction that these losses must be proper (see \cref{sec:optimal-routing} for a definition). While standard functions like Brier or Cross-Entropy are (strictly) proper out-of-the-box, other losses such as the classification loss are not, and should not be evaluated directly on raw probability predictions. To handle this, some of the losses below couple the underlying decision cost with an optimal post-processing step (such as reweighting or thresholding the predicted probabilities $p$). This ensures that the expected loss minimized by the model continues to act as a proper loss with respect to the original predicted probabilities.

\begin{itemize}
    \item Square (Brier) Loss: This denotes the typical squared distance between the label distribution and prediction, $\ell(y, p) = \|y - p\|_2^2$.
    \item Cross Entropy Loss: This is the standard negative log-likelihood, defined as $\ell(y, p) = -\sum_{c} y_c \log(p_c)$.
    \item Classification Loss: Let $\hat{y} = \arg\max_c p_c$ be the predicted class with maximal probability. The expected 0-1 error given the true probabilities $y$ is $\ell(y, p) = 1 - y_{\hat{y}}$.
    \item Weighted False-Positive Loss (``Fn Fp Loss'', Binary only): This applies asymmetric penalties to false positives ($c_{FP}$) and false negatives ($c_{FN}$). The decision $\hat{y} \in \{0, 1\}$ is determined by thresholding the predicted odds: $\hat{y} = \mathbf{1}[\frac{p_1}{p_0} \ge \frac{c_{FP}}{c_{FN}}]$. The expected loss is evaluated against the true probabilities $y$ as $\ell(y, p) = c_{FP} \cdot \mathbf{1}[\hat{y}=1] y_0 + c_{FN} \cdot \mathbf{1}[\hat{y}=0] y_1$.
    \item Three Part Loss (binary only): This loss extends the weighted false positive loss by introducing an alternative third action with a fixed cost. Let $p_1$ be the predicted probability of the positive class. The loss is defined conditionally: $\ell(y, p) = y_1$ if $p_1 < 0.25$; a constant $0.25$ if $0.25 \le p_1 < \frac{15}{16}$; and $4 y_0$ if $p_1 \ge \frac{15}{16}$.
    \item Asymmetric Class Penalty (multiclass): This applies a penalty multiplier $\gamma$ strictly when predicting class 0. The decision scores $s$ are biased against class 0 such that $s_0 = \gamma p_0 + (1-\gamma)$, while $s_c = p_c$ for $c \neq 0$. For the optimal action $\hat{y} = \arg\max_c s_c$, the loss is $\gamma(1 - y_0)$ if $\hat{y} = 0$, and standard $1 - y_{\hat{y}}$ for all other classes.
\end{itemize}

\subsection{Synthetic Data}
\compareplotsvertsmall
    {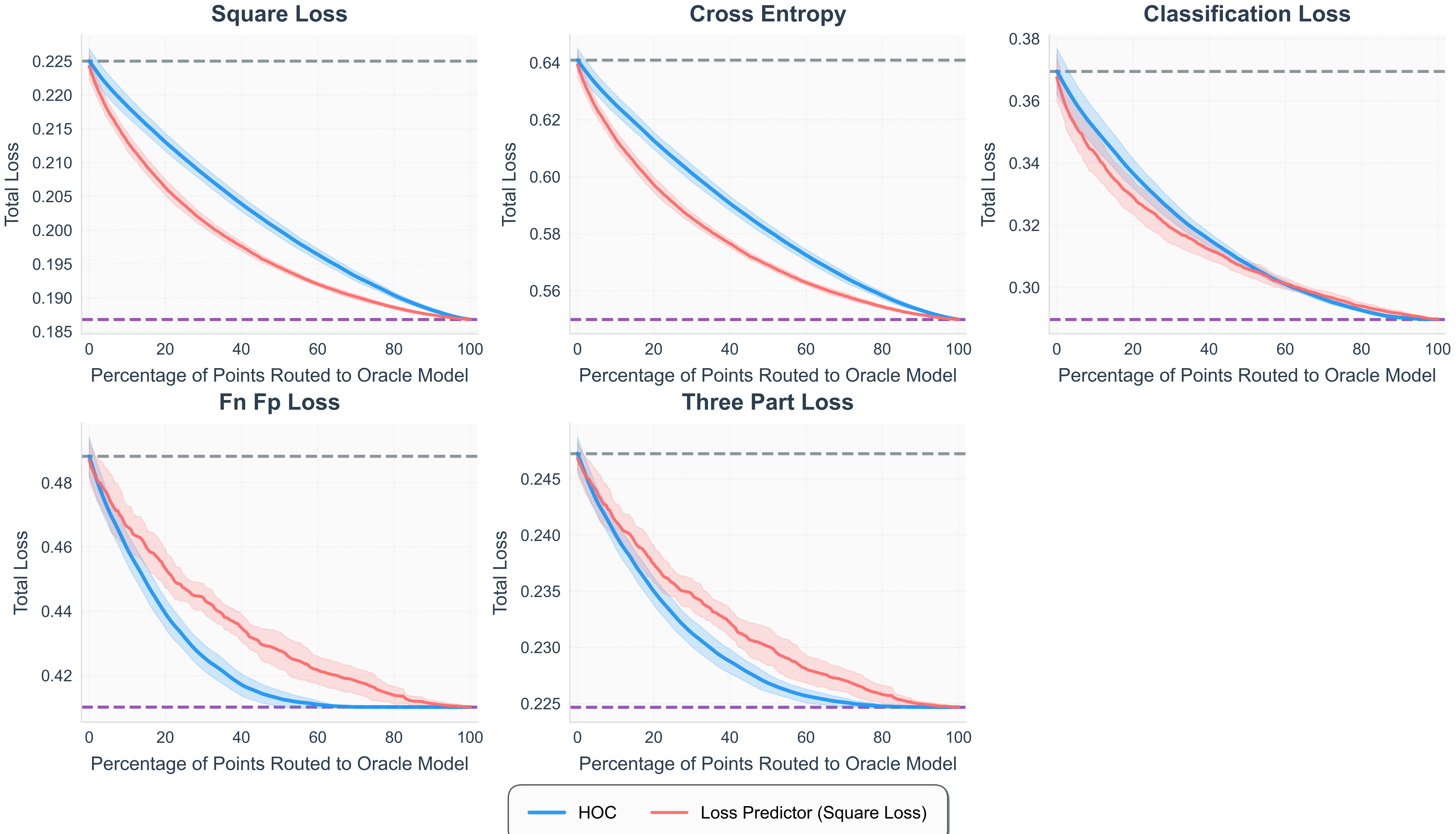} %
    {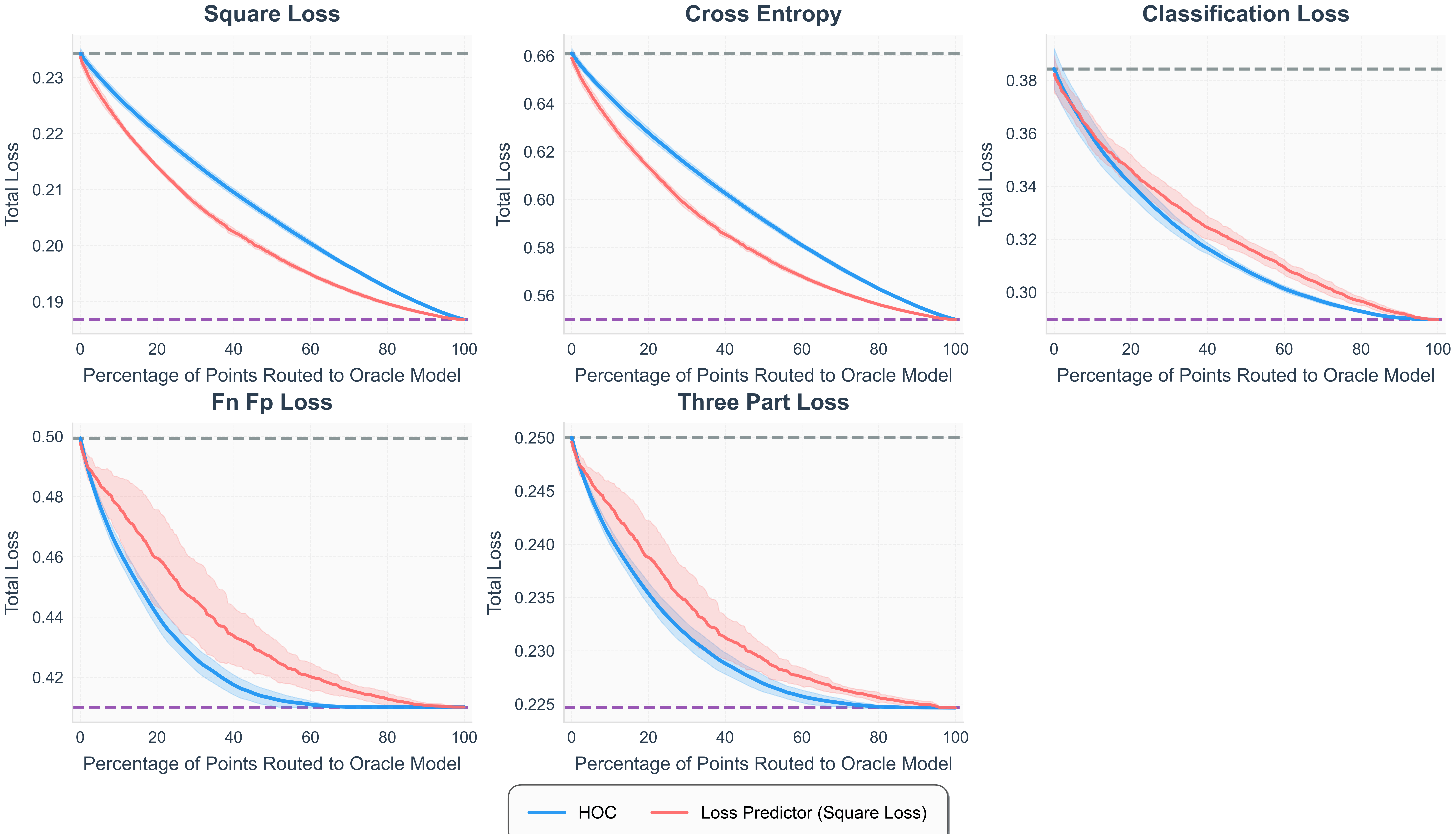} %
    {fig:loss_johnson_no_pp} %
    {fig:loss_johnson_recal} %
    {Routing performances for different target loss functions for synthetic data as described in \cref{sec:synthetic_data_gen} and target function as depicted in \cref{fig:johnson_fn}. The loss predictor baseline is a static loss predictor trained to predict the amount of reducible square loss. The right (b) employs an out-of-the-box weak predictor for the dataset, whereas the left (a) uses the same predictor but with an additional post-hoc calibration step on the same buckets used by the higher-order-calibrated router. Error bars show the standard deviation over 10 runs, re-randomizing both the training set used to train the weak model and the calibration set used to higher-order calibrate and train a loss predictor.} %
    {fig:loss_johnson_full}

\subsection{CIFAR-10H}
\compareplotsvert
    {loss_cifar10_recal.png} %
    {loss_cifar10_no_pp.png} %
    {fig:loss_cifar10_no_pp} %
    {fig:loss_cifar10_recal} %
    {Routing performances for different target loss functions for the CIFAR-10H dataset. The loss predictor baseline is a static loss predictor trained to predict the amount of reducible cross-entropy loss. The right (b) employs an out-of-the-box weak predictor for the dataset, whereas the left (a) uses the same predictor but with an additional post-hoc calibration step on the same buckets used by the higher-order-calibrated router.} %
    {fig:loss_cifar10_full}

\subsection{SNLI}
\compareplotsvert
    {loss_snli_recal.png} %
    {loss_snli_no_pp.png} %
    {fig:loss_snli_no_pp} %
    {fig:loss_snli_recal} %
    {Routing performances for different target loss functions for the SNLI dataset. The loss predictor baseline is a static loss predictor trained to predict the amount of reducible cross-entropy loss. The right (b) employs an out-of-the-box weak predictor for the dataset, whereas the left (a) uses the same predictor but with an additional post-hoc calibration step on the same buckets used by the higher-order-calibrated router.} %
    {fig:loss_snli_full}

\subsection{ChaosNLI}
\compareplotsvert
    {loss_chaosnli_recal.png} %
    {loss_chaosnli_no_pp.png} %
    {fig:loss_chaosnli_no_pp} %
    {fig:loss_chaosnli_recal} %
    {Routing performances for different target loss functions for the ChaosNLI dataset. The loss predictor baseline is a static loss predictor trained to predict the amount of reducible cross-entropy loss. The right (b) employs an out-of-the-box weak predictor for the dataset, whereas the left (a) uses the same predictor but with an additional post-hoc calibration step on the same buckets used by the higher-order-calibrated router.} %
    {fig:loss_chaosnli_full}
\section{Plots of Three-Way Decisions}\label{sec:three-way-plots}

\begin{figure}[H]
     \centering
     \begin{subfigure}[b]{0.48\textwidth}
         \centering
         \includegraphics[width=\textwidth]{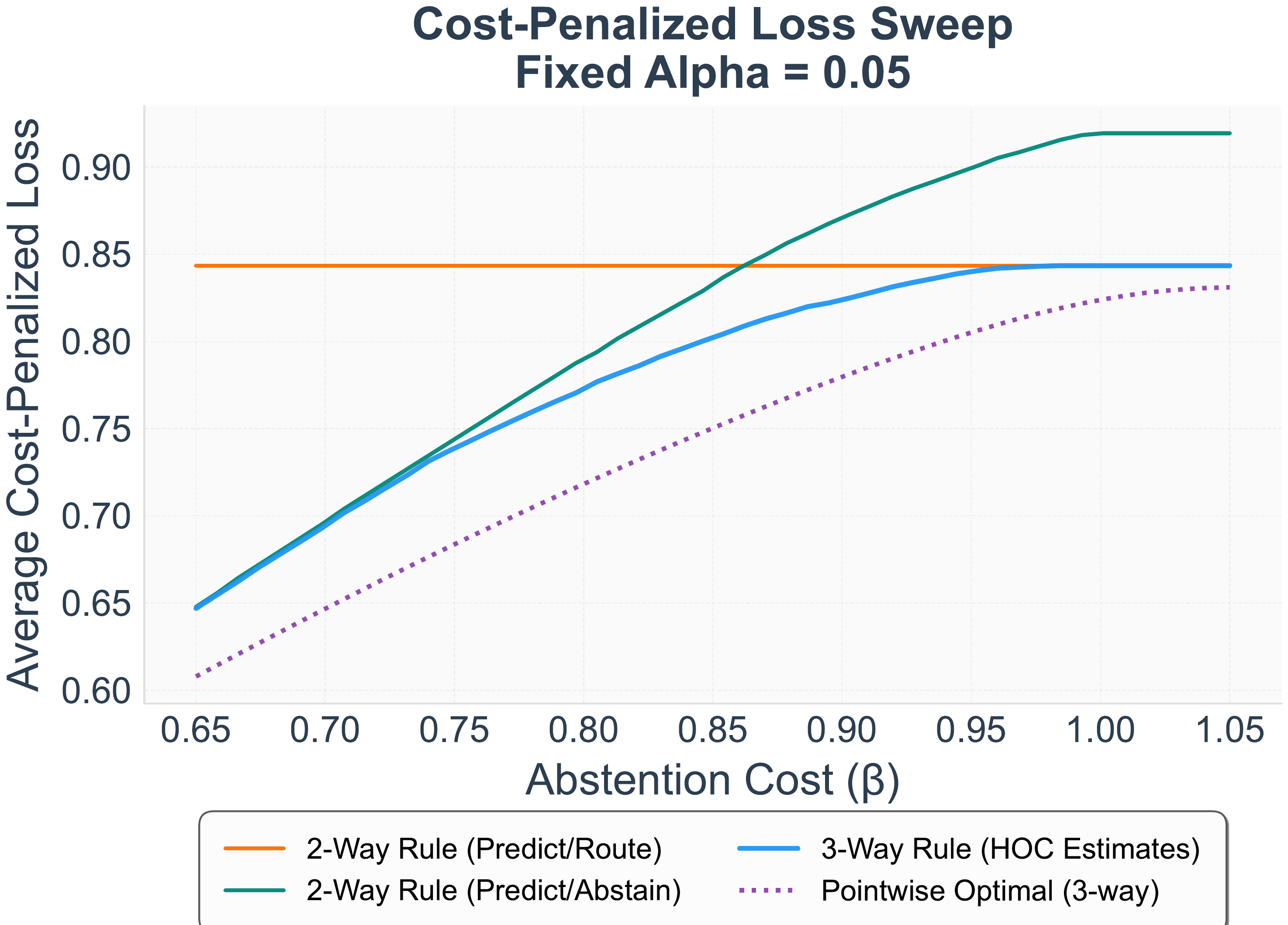}
         \caption{Synthetic Data (target function from \cref{fig:johnson_fn})}
         \label{fig:three-way-synthetic}
     \end{subfigure}
     \hfill
     \begin{subfigure}[b]{0.48\textwidth}
         \centering
         \includegraphics[width=\textwidth]{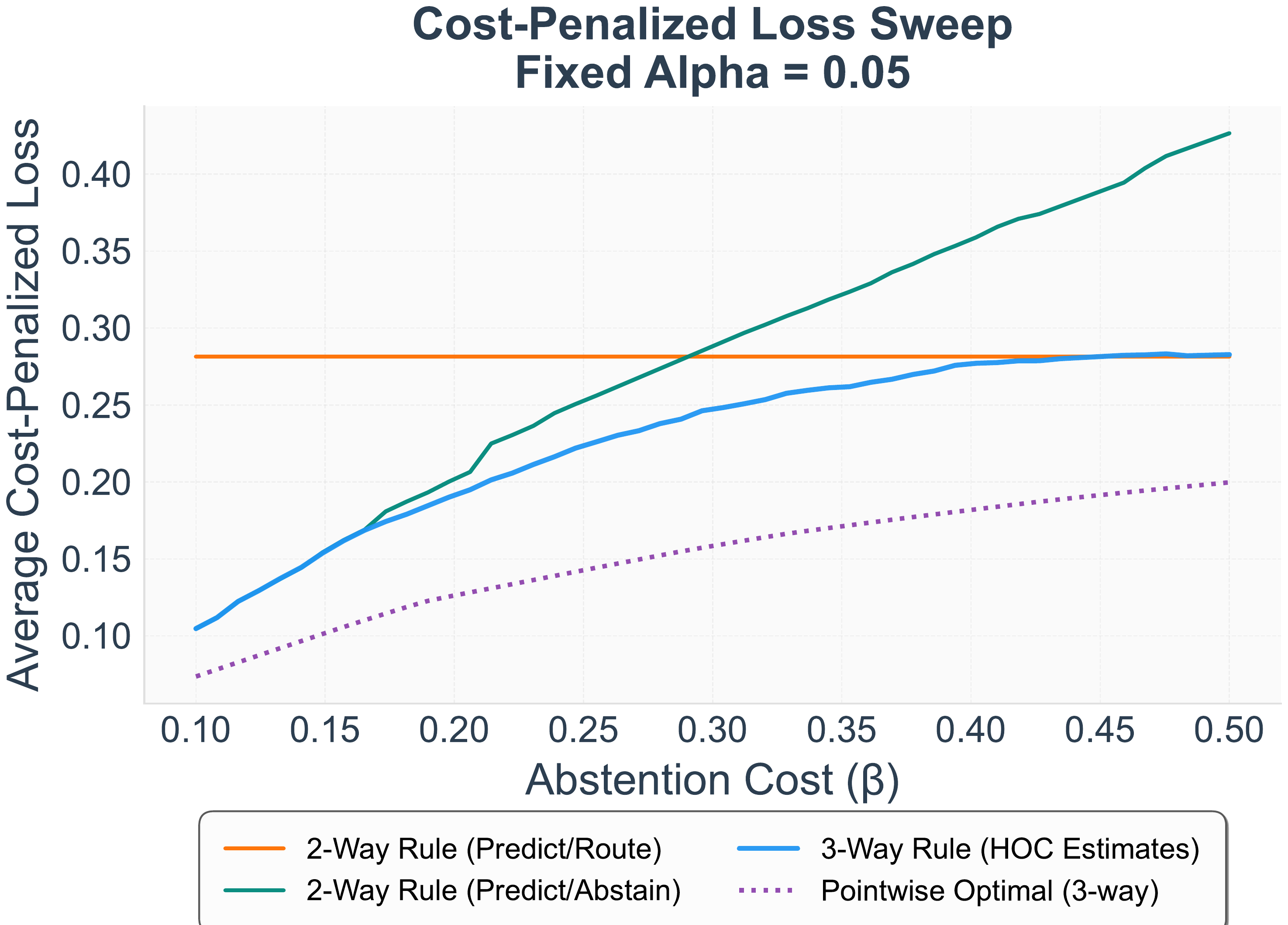}
         \caption{CIFAR-10H}
         \label{fig:three-way-cifar}
     \end{subfigure}

     \vspace{10pt} %

     \begin{subfigure}[b]{0.48\textwidth}
         \centering
         \includegraphics[width=\textwidth]{snli_cost_sweep.pdf}
         \caption{SNLI}
         \label{fig:three-way-snli}
     \end{subfigure}
     \hfill
     \begin{subfigure}[b]{0.48\textwidth}
         \centering
         \includegraphics[width=\textwidth]{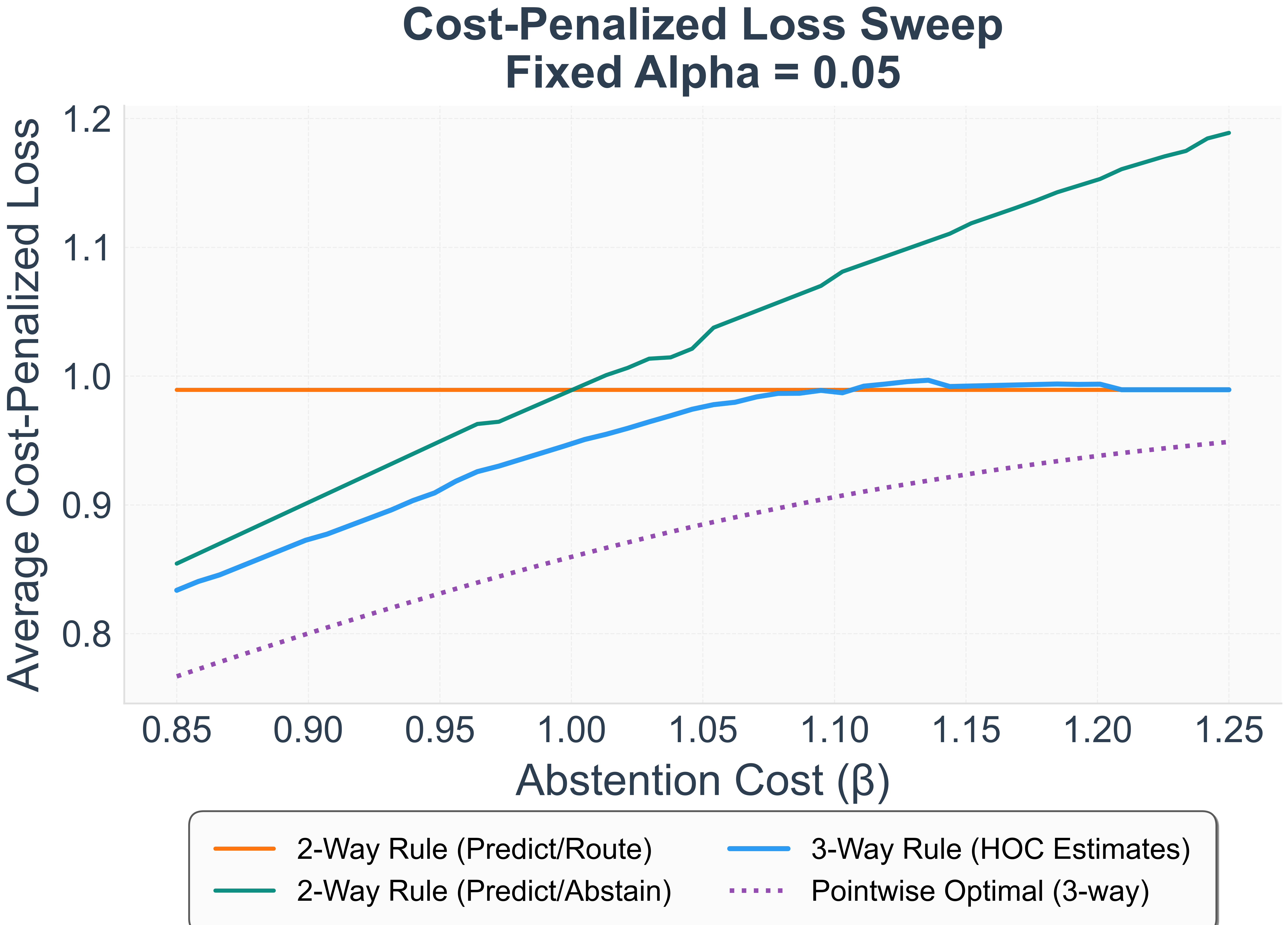}
         \caption{ChaosNLI}
         \label{fig:three-way-chaos}
     \end{subfigure}
     
     \caption{Depiction of our method's cost-penalized performance on the three-way decision task of Predict vs Route vs Abstain across four datasets, for a fixed routing penalty of $\alpha = 0.05$ while sweeping the abstention penalty $\beta$. We see that our method gracefully adapts to the three-way decision setting, out-performing or closely aligning with the two-way decision baselines of predict/route (orange) and predict/abstain (green).}
     \label{fig:combined-three-way-routing}
\end{figure}

\end{document}